\documentclass[sigconf]{acmart}

\AtBeginDocument{%
  }

\copyrightyear{2024}
\acmYear{2024}
\setcopyright{acmlicensed}
\acmConference[KDD '24] {Proceedings of the 30th ACM SIGKDD Conference on Knowledge Discovery and Data Mining }{August 25--29, 2024}{Barcelona, Spain.}
\acmBooktitle{Proceedings of the 30th ACM SIGKDD Conference on Knowledge Discovery and Data Mining (KDD '24), August 25--29, 2024, Barcelona, Spain}
\acmISBN{979-8-4007-0490-1/24/08}
\acmDOI{10.1145/3637528.3671920}
\usepackage{natbib}

\usepackage[utf8]{inputenc} %
\usepackage[T1]{fontenc}    %
\usepackage{hyperref}       %
\usepackage{url}            %
\usepackage{booktabs}       %
\usepackage{amsfonts}       %
\usepackage{nicefrac}       %
\usepackage{microtype}      %
\usepackage{xcolor}         %

\usepackage{color-edits}
\usepackage{enumerate}
\usepackage{cite}
\usepackage{enumitem}
\usepackage{color}
\usepackage{graphicx}
\usepackage{mathtools}
\usepackage{caption}
\usepackage{subcaption}
\usepackage{multirow}

\usepackage{amssymb,amsthm,bm,bbm}
\usepackage{amsmath}
\usepackage{algorithm}
\usepackage{float}
\usepackage[noend]{algpseudocode}

\newtheorem{proposition}{Proposition}

\addauthor{jiajie}{blue}
\addauthor{jerry}{red}

\usepackage{xr-hyper}

\begin{document}

\title{Robust Predictions with Ambiguous Time Delays: A Bootstrap Strategy}

\author{Jiajie Wang}
\affiliation{%
  \institution{Changsha Research Institute of Mining and Metallurgy}
  \city{Changsha}
  \country{China}}
\email{wangjj6@minmetals.com}

\author{Zhiyuan Jerry Lin}
\authornote{The work is performed during the author's personal capacity.}
\affiliation{%
  \institution{Meta}
  \city{Menlo Park}
  \country{USA}}
\email{zylin@meta.com}

\author{Wen Chen}
\authornote{Corresponding Author.}
\affiliation{
  \institution{Changsha Research Institute of Mining and Metallurgy}
  \city{Changsha}
  \country{China}
}
\email{wenchen@minmetals.com}

\renewcommand{\shortauthors}{Jiajie Wang, Zhiyuan Jerry Lin, \& Wen Chen}

\begin{abstract}
In contemporary data-driven environments, the generation and processing of multivariate time series data is an omnipresent challenge, often complicated by time delays between different time series. These delays, originating from a multitude of sources like varying data transmission dynamics, sensor interferences, and environmental changes, introduce significant complexities. Traditional Time Delay Estimation methods, which typically assume a fixed constant time delay, may not fully capture these variabilities, compromising the precision of predictive models in diverse settings.

To address this issue, we introduce the Time Series Model Bootstrap (TSMB), a versatile framework designed to handle potentially varying or even nondeterministic time delays in time series modeling. Contrary to traditional approaches that hinge on the assumption of a single, consistent time delay, TSMB adopts a non-parametric stance, acknowledging and incorporating time delay uncertainties. TSMB significantly bolsters the performance of models that are trained and make predictions using this framework, making it highly suitable for a wide range of dynamic and interconnected data environments.

Our comprehensive evaluations, conducted on real-world datasets with different types of time delays, confirm the adaptability and effectiveness of TSMB in multiple contexts. These include, but are not limited to, power and occupancy forecasting in intelligent infrastructures, air quality monitoring, and intricate processes like mineral processing. Further diagnostic analyses strengthen the case for the TSMB estimator's robustness, underlining its significance in scenarios where ambiguity in time delays can have a significant impact on the predictive task.

\end{abstract}


\maketitle

\section{Introduction}
\label{sec:intro}
In today's rapidly evolving data-centric landscape, a myriad of devices and sensors are seamlessly integrated, creating a vast and intricate system that continuously generates multivariate time series data. This rich data source is important for advancing data mining applications, enhancing predictive modeling, and refining data-driven decision-making processes across various fields, as underlined by contemporary research \citep{lin2021probability, brunner2021challenges}.

However, when employing this data for predictive analysis, practitioners encounter a significant hurdle: time delays in data collection and transmission. These delays are not confined to any specific domain but are a widespread challenge affecting numerous sectors. For instance, in urban monitoring systems, sensors disseminated throughout a city to measure air quality might report data at different times due to diverse factors like transmission paths or environmental disturbances. This results in data from one sensor arriving with a delay relative to another, complicating real-time air quality predictions for a given area.
These discrepancies, inherent in the sequential and intricate nature of modern industrial, urban, and technological environments, render the raw time series data more complex and challenging to analyze directly.

To elucidate this phenomenon, consider a hypothetical smart city environment depicted in Figure~\ref{fig:illustration}:
\begin{figure}[ht]
  \centering
  \includegraphics[width=0.8\linewidth]{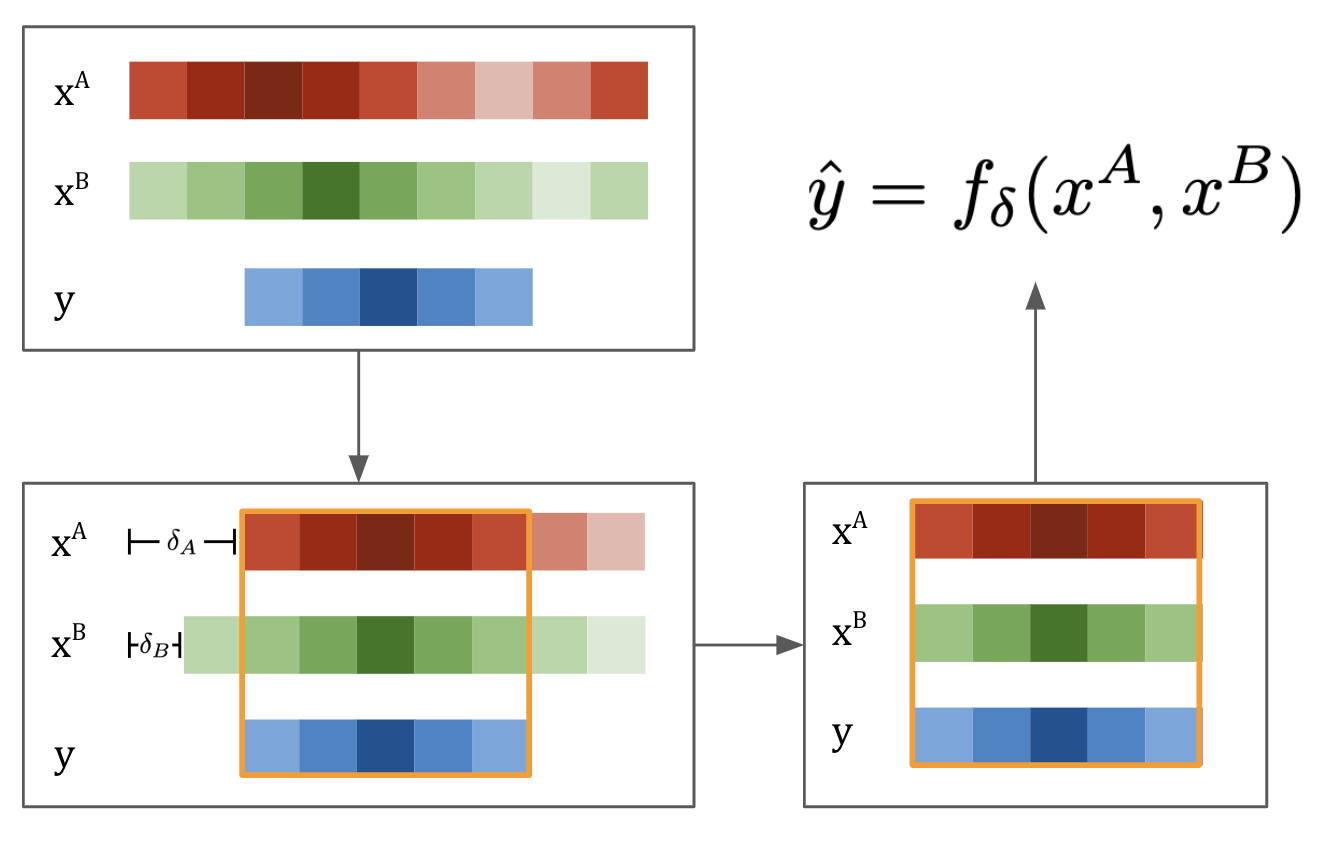}
  \caption{An illustration of the challenge of time delay estimation in multivariate time series. 
  In this scenario, $\bm{\delta}=\{\delta_A, \delta_B\}$ are well-defined and have unique values, 
  which isn't always the case in real-world applications involving unpredictable events and fluctuating noise.}
  \label{fig:illustration}
\end{figure}
Imagine a smart city scenario where sensors A and B are deployed to monitor the traffic flow and air quality respectively. 
Sensor A captures the volume of vehicles at a particular intersection, 
while Sensor B measures the speed of wind in a nearby park.
The data captured by these sensors starting at time $t$ over a span of time $m$ can be represented as two vectors:
\begin{align*}
    \bm{x}^A(t)=[x^{A}_t, x^{A}_{t+1}, \dots, x^{A}_{t+m-1}]^T\\
    \bm{x}^B(t)=[x^{B}_t, x^{B}_{t+1}, \dots, x^{B}_{t+m-1}]^T.
\end{align*}
We are interested in monitoring and predicting the air quality of a residential area in the neighborhood for the same period of time, denoted as
\begin{align*}
    \bm{y}(t) = [y_t, y_{t+1}, \dots, y_{t+m-1}]^T.
\end{align*}

Given the dynamic nature of urban environments, different wind conditions and events at the intersection, such as a traffic jam, can affect air quality after a certain delay due to the dispersion of pollutants.
Directly utilizing the raw data from sensors A and B to predict $\bm{y}(t)$ can lead to inaccuracies, 
given the inherent delay between the traffic situation and its subsequent impact on air quality.
By adjusting the data streams to 
\begin{align*}
\bm{x}^A(t+\delta_A) = [x^{A}_{t+\delta_A}, x^{A}_{t+\delta_A+1}, \dots, x^{A}_{t+\delta_A+m-1}]^T\\
\bm{x}^B(t+\delta_B) = [x^{B}_{t+\delta_B}, x^{B}_{t+\delta_B+1}, \dots, x^{B}_{t+\delta_B+m-1}]^T,
\end{align*}
we ensure proper alignment and synchronization of the sensor readings.
Thus, a predictive model can then formulate $y_t$ more accurately based on the time-aligned readings from both sensors:
$$\hat{y}_t = f(x^{A}_{t+\delta_A}, x^{B}_{t+\delta_B}).$$
Estimating such delays, termed the \textit{time delay estimation (TDE)} problem~\citep{Knapp1976-sr, chen2004time, bjorklund2003review}, is pivotal and there is a rich collection of literature, which we discuss in Section~\ref{sec:related_work}.
Existing TDE techniques typically select the time delay vector $\bm{\delta}$ to be the one maximizing a specific score function, such as cross-correlation or mutual information~\citep{Knapp1976-sr, Mars1982-nw, He2013-ne, Radhika2014-ob}.

However, one fundamental assumption most existing TDE methods rely on is that there exists an unique, constant-valued time delay, which often does not hold in many real-world applications.
Taking our motivating air quality prediction task as an example: during peak traffic hours, the delay between increased vehicle emissions and reduced air quality might vary depending on wind speed and direction, which itself is subject to sudden changes due to weather patterns. Similarly, the wind speed measured in the park does not consistently influence air quality in the same manner; its impact can be delayed or altered by urban structures, green spaces, and atmospheric conditions. These varying delays mean that the time lag between cause (status of the pollutants reflected by sensor readings) and effect (air quality) is not fixed but fluctuates in a manner that traditional time delay estimation methods assuming deterministic delays struggle to accurately capture.
Consequently, accurately predicting air quality requires an approach capable of adapting to the inherently unpredictable nature of these time delays, emphasizing the need for new modeling techniques that can account for such variability.

The situation becomes even more complicated when the end goal is not to estimate the time delay per se, but to use this information to construct accurate predictive models (e.g., for air quality prediction).
Even if one is able to obtain a probabilistic representation of the time delay --- a challenge already faced by classic TDE methods --- it is not immediately clear how we can use this information to improve off-the-shelf machine learning models' performance on this prediction task without making assumptions about the predictive models' architectures.
While TDE plays a pivotal role in multivariate time series modeling, it is crucial to underscore that in many applications, the overarching goal is to maximize the performance of downstream predictive tasks.
The actual time delays, although important for model alignment, are often secondary in significance.
The precise formulation of this problem is detailed in Section~\ref{sec:tsmb}.

In this work, we present a novel framework Time Series Model Bootstrap (TSMB) to help with multivariate time series modeling in the presence of potentially non-deterministic time delays.
TSMB does not require explicit assumptions about the nature of these delays and is designed to integrate seamlessly with any black-box predictive model, offering straightforward implementation and an inherent statistical interpretation.

Finally, it is worth noting that while some classical time series models~\citep{brockwell2002introduction} and more recent neural network models~\citep{salinas2020deepar, vaswani2017attention, goodfellow2016deep, lim2021temporal}
in theory have the capacity of learning and incorporating time delay information automatically, directly accounting for the (potentially stochastic) time delay, such as our proposed method TSMB, can further bolster their predictive performance as we later show in Section~\ref{sec:experiment}.

In summary, in this paper
we spotlight the often-overlooked issue of possibly non-deterministic time delays in misaligned multivariate time series modeling. Our approach not only addresses the complexities associated with non-deterministic time delays but also demonstrates that effectively managing these delays can significantly enhance models' predictive performance.
We then introduce TSMB, Time Series Model Bootstrap, an innovative framework adept at handling both stochastic and fixed time delays prevalent in real-world datasets, without making explicit assumptions about the form of such time delays.
Designed to work with any black-box predictive model, TSMB can be easily implemented while having a natural statistical interpretation.
Finally, we empirically showcase TSMB's superior performance over classic TDE methods across a range of real-world predictive tasks with various time delay types using both time-series transformer and classical machine learning models.
Additionally, we delve deep into TSMB's characteristics, shedding light on aspects like prediction coverage and the nuances of time delay estimation.

\section{Related Work}
\label{sec:related_work}

\textbf{Time delay estimation}
Time series analysis is a classic field of study in data mining, statistical learning, and data analysis. Time delay estimation, or TDE, is a problem of estimating the relative time difference between different streams of signals, often being presented in the form of multivariate time series.
TDE is widely used in multivariate time series datasets with sequential structures between time series such as ones from industrial processes, seismology, acoustics, and communication.
The most popular TDE method is the \textit{generalized cross-correlation (GCC)}, originally proposed by \citet{Knapp1976-sr}, whose core idea is to identify a time delay to maximize the cross-correlation between two time series. This method has been extensively studied and proven to work well in reasonably noisy environments~\citep{ianniello1982time, champagne1996performance} and has been extended, for example, to adapt special noise types such as reverberation in acoustics~\citep{Benesty2004-gu, chen2003robust}.
Another similar idea is to use \textit{time-delayed mutual information (TDMI)}~\citep{Mars1982-nw, albers2012using}, instead of cross correlation, which is expected to perform better in non-linear systems.
In practice, mutual information can be estimated non-parametrically using k-nearest neighbor distances~\citep{Kraskov2004-li,Ross2014-kd}.
When more than two time series are present, instead of estimating the time delay separately for each time series, methods using joint mutual information (also known as non-mutual information methods) are employed~\citep{Ruan2014-kj, He2013-ne, Radhika2014-ob}. There are other methods using PCA~\citep{Chen2020-pm}, random walk~\citep{Ohira1997-br}, and Wasserstein distance~\citep{nichols2019time}.
Despite these methods' claimed improved performance, many of them are computationally expensive and prohibitively slow when there are more than a few time delays to be estimated.
Generalized cross-correlation and mutual information remain the most widely used and proven to be robust among other methods.
Additionally, TDE is also known as time delay identification or time delay signature extraction in the field of communication and security~\citep{Rontani2009-uc}, and relatedly, its reverse problem time delay concealing is also being extensively studied~\citep{Gao2021-gg, li2013hybrid, xiang2016suppression, han2020generation}.

\textbf{Sequence alignment}
Separately, \textit{sequence alignment} is another related stream of research widely used and studied in areas like bioinformatics. Dynamic programming algorithms such as dynamic time warping (DTW)~\citep{bellman1959adaptive, muller2007dynamic} and the Needleman-Wunsch algorithm~\citep{likic2008needleman} are often used to calculate global alignment while local alignment algorithms like Smith-Waterman algorithm~\citep{xia2022review} can be used for more efficient alignment at the expense of potentially sub-optimal matching.
However, sequence alignment algorithms usually seek for exact matching between two sequences, and in the problem concerned in this paper, such requirement is generally not met, rendering this class of methods not applicable here.
There are also tensor-completion-based methods such as \citet{qian2021multi} and \citet{liu2019costco} applicable to spatial-temporal data, but their aim is to rectify missing or inaccurate data, which is not the setting this paper is focusing on.

\textbf{Time Series Bootstrap}
Bootstrap is a resampling method that approximates the true population's distribution using random sampling with replacement.
Bootstrap is often used to quantify uncertainty around sample estimates.
One implicit assumption bootstrap methods make is that the sampled data is identically and independently distributed (IID), which is obviously violated when it comes to time series data.
To overcome this problem, several methods are commonly used to perform bootstrap on time series data.
Possibly the most widely used time series bootstrap method is block bootstrap~\citep{lahiri1999theoretical, kunsch1989jackknife}, whose central idea is to sample continuous segments (``blocks'') of the time series with replacement.
This method will maintain local sequential dependency but not necessarily global structure.
Another popular choice is sieve bootstrap~\citep{buhlmann1997sieve, andre2002forecasting}.
Instead of resampling the data itself, new data is generated using an autoregressive model and individual residuals are sampled from the data.
Other approaches and variants include the wild bootstrap~\citep{wu1986jackknife} and stationary bootstrap~\citep{politis1994stationary}.

\section{Time Series Model Bootstrap}
\label{sec:setup_and_method}
In this section, we formally present the problem we study in this paper.
We then describe our proposed method TSMB
in detail.

\subsection{Problem Definition}
\label{sec:setup}
The fundamental problem we are concerned about in this paper is multivariate time series modeling with (potentially stochastically) misaligned time series data.
To formally define the problem, assume we observe a dataset $\mathcal{D} = \{\mathcal{X}, \bm{y}\}$ consisting of a feature matrix $\mathcal{X}$ and a target vector of interest $\bm{y}$.
Vector $\bm{y} = [y_t, y_{t+1}\dots, y_{t+m-1}]^T$ has a starting time $t$ and length of $m$.
Examples of $\bm{y}$ include whether a room is currently occupied in a smart home setting (a classification task) or the particle size of the concentrate in mineral processing (a regression task).
The feature matrix $\mathcal{X} = [\bm{x}^1, \bm{x}^2, \dots, \bm{x}^n]$ represents other observed multivariate time series data where each individual time series $\bm{x}^i = [x^i_0, x^i_2, \dots, x^i_{M-1}]^T$ is a sequence of recorded variables (e.g., measurements over time from a sensor) for some integer $M \geq m$.
Each $\bm{x}^i$ in $\mathcal{X}$ may be misaligned with the target vector $\bm{y}$ differently.
While classic TDE methods assume such misalignment is $\bm{x}^i$ being shifted by a fixed value, in practice, $\bm{x}^i$ can be misaligned with $\bm{y}$ in arbitrary ways.
The goal here is to construct a predictive model trained on $\mathcal{D}$ that can accurately predict out-of-sample $\bm{y'}$ with corresponding $\mathcal{X}'$.

In a well-behaved misaligned time series data, there are (unobserved) delays $\bm{\delta} = \{\delta_1, \delta_2, \dots, \delta_n\}$ between each $\bm{x}^i$ and $\bm{y}$ such that element $x^i_{t+\delta_i}$ would correspond to element $y_t$ in perfectly aligned data.
For example, $\bm{x}^i$ could be a grinding mill's power at the beginning of a mineral processing workflow and $y_t$ is the particle size of the concentrate (i.e., the granularity of the refined mineral of interest) measured at the end of the workflow.
Without loss of generality, it is safe to assume $0 \leq t < t+m \leq M$ and $\delta_i, t, m, M \in \mathbb{Z^+}$ where $\mathbb{Z^+}$ is the set of all positive integers.
Then given a family of machine learning models, we would like to construct a time-aligned dataset $\mathcal{D}_{\bm{\delta}} = \{\mathcal{X}(\bm{\delta}), \bm{y}\}$ and a predictive model $f_{\bm{\delta}}$ such that it can accurately and robustly predict some future $\bm{y}'$.
With traditional TDE methods, matrix $\mathcal{X}(\bm{\delta})$ is typically constructed as
$\mathcal{X}(\bm{\delta}=\bm{\hat{\delta}}) = [\bm{x}^1(t+\hat{\delta}_1), \bm{x}^2(t+\hat{\delta}_2), \dots, \bm{x}^n(t+\hat{\delta}_n)]$
where $\bm{\hat{\delta}} = \{\hat{\delta}_1, \hat{\delta}_2, \dots, \hat{\delta}_n\}$ is the inferred time delay vector
and $\bm{x}^i(t+\hat{\delta}_i) = [x^i_{t+\hat{\delta}_i}, x^i_{t+\hat{\delta}_i+1}\dots, x^i_{t+\hat{\delta}_i+m-1}]^T$ is a shifted subsequence of the original time series $\bm{x}^i$ correcting for the estimated time delay.
In classic TDE methods, $\bm{\hat{\delta}}$ is obtained by maximizing some score function $S(\mathcal{D}_{\bm{\hat{\delta}}})$.
Common choices of the score function $S$ include GCC and TDMI.

However, the time-varying noise in some data may undermine our capability of accurately identifying the time delays using classic TDE methods.
More importantly, in many real-world problems with complex dynamics and chaotic systems (e.g., smoke or ocean turbulence in fluid mechanics or grinding processes in mineral processing), instead of having a unique, identifiable value, the time delay $\bm{\delta}$ may be a random variable itself.
In these scenarios, both the time delay estimation and the downstream prediction modeling task can become very challenging for classic TDE methods as their fundamental assumption that there exists a unique time delay vector $\bm{\delta}$ is violated.

We note that although the problem formulation resembles the classic time delay estimation problem, it differs in that the estimation of time delay is not the end goal.
Instead, in this problem we call \textit{misaligned multivariate time series modeling}, we aim to use the estimated time delay as auxiliary information to improve the performance of the predictive model for some downstream prediction task.

\subsection{Method}
\label{sec:tsmb}

Now we introduce our proposed framework \emph{Time Series Model Bootstrap (TSMB)}, for misaligned multivariate time series modeling.
Model prediction obtained via this framework is referred to as the \textit{TSMB estimator}.

Most traditional TDE methods implicitly assume that there exists a unique value of time delay $\bm{\delta}$ represented by its point estimate $\bm{\hat{\delta}}$.
When we build a machine learning model with aligned dataset $\mathcal{D}_{\bm{\hat{\delta}}} = \{\mathcal{X}(\bm{\hat{\delta}}), \bm{y}\}$, 
the model is technically estimating $\mathbb{E}[Y|X, \bm{\delta}=\bm{\hat{\delta}}]$ instead of $\mathbb{E}[Y|X]$, the typical desideratum of machine learning tasks.
This is a fine approximation when there \emph{does} exist a unique $\bm{\delta}$ that we can accurately estimate.
When this condition does not hold, which is likely in many real-world scenarios with complex dynamics,
the prediction $\hat{y} = f_{\bm{\hat{\delta}}}(x) = \mathbb{E}[Y|X=x, \bm{\delta}=\bm{\hat{\delta}}]$ may be a biased estimator of the real estimand $\mathbb{E}[Y|X]$ when evaluated on new, out-of-sample data.

\begin{algorithm}[ht]
  \caption{Time series model bootstrap (TSMB)}
  \label{alg:tsmb}
  \begin{algorithmic}
    \Require $\mathcal{D}_{train}=\{\mathcal{X}_{train}, \bm{y}_{train}\}$
    \Require $\mathcal{D}_{test}=\{\mathcal{X}_{test}, \bm{y}_{test}\}$
    \Require Score function $S$
    \State \texttt{\# Model training with TSMB}
    \State \texttt{model\_list} $\gets \emptyset$
    \State \texttt{td\_list} $\gets \emptyset$
    \For{\texttt{b = 1..B}}
    \State Draw block bootstrap sample $\mathcal{D}^b$ from $\mathcal{D}_{train}$
    \State $\bm{\hat{\delta}}^b = \text{argmax}_{\bm{\delta}} S(\bm{\delta}; \mathbf{D}_b)$
    \State Fit $f_{\bm{\hat{\delta}}^b}$ on data $\mathcal{D} = \{\mathcal{X}_{train}(\bm{\hat{\delta}}^b), \bm{y}_{train}\}$
    \State \texttt{td\_list} $\gets$ \texttt{td\_list.append}$(\bm{\hat{\delta}}^b)$
    \State \texttt{model\_list} $\gets$ \texttt{model\_list.append}$(f_{\bm{\hat{\delta}}^b})$
    \EndFor

    \State \texttt{\# Making predictions with TSMB}
    \State $\bm{y}_{pred} \gets \bm{0}$
    \For{\texttt{b = 1..B}}
    \State $\bm{\hat{\delta}}^b \gets$ \texttt{td\_list[b]}
    \State $f_{\bm{\hat{\delta}}^b} \gets$ \texttt{model\_list[b]}
    \State $\bm{y}_{pred, b} \gets f_{\bm{\hat{\delta}}^b}(\mathcal{X}_{test}(\bm{\hat{\delta}}^b))$
    \State $\bm{y}_{pred} \gets \bm{y}_{pred} + \bm{y}_{pred, b}$
    \EndFor
    \State $\bm{y}_{pred}  \gets \bm{y}_{pred} / B$\\
    \Return $\bm{y}_{pred}$
  \end{algorithmic}
\end{algorithm}

While we might not be able to pinpoint an exact time delay value under the aforementioned scenarios, it is still possible to treat $\bm{\delta}$ as a random variable and describe its value as a probability distribution.
By obtaining a bootstrap sample $\mathcal{D}^b$ (e.g., via block bootstrap~\citep{kunsch1989jackknife, lahiri1999theoretical}) from the original dataset $\mathcal{D}$,
we are able to attain a sample of $\bm{\hat{\delta}}^b$ by maximizing the score function $S$ on each individual bootstrap dataset sample $\mathcal{D}^b$.
$\bm{\hat{\delta}}^b$ can be regarded as a sample drawn from the bootstrap time delay distribution $\mathcal{F}^B_{\bm{\delta}}$, which in general can be treated to be a reasonable approximation of the true underlying time delay distribution $\mathcal{F}_{\bm{\delta}}$, as is commonly assumed in bootstrap methods.
With a $B$ bootstrap time delay samples, we have the empirical bootstrap distribution of $\bm{\delta}$.
Recall that given a fixed dataset, the predictive model is a function of $\bm{\delta}$ and we are able to fit a model $f_{\bm{\hat{\delta}}^b}(x)$ to approximate $\mathbb{E}[Y|X=x, \bm{\delta}=\bm{\hat{\delta}}^b]$ similar to how we do predictive modeling with traditional TDE methods, but with a bootstrap time delay $\bm{\hat{\delta}}^b$.
After fitting and obtaining predictions from each of these bootstrapped models, we are able to calculate our prediction $\hat{y}$ by averaging over $B$ different bootstrap models.
Algorithm~\ref{alg:tsmb} describes the pseudocode to fit models and make predictions with TSMB.

By treating $\bm{\delta}$ as a random variable, obtaining the model prediction by integrating out $\bm{\delta}$ is the optimal decision-theoretic approach.
Because we obtain model's predictive samples based on time series bootstrap,
we refer to this method as Time Series Model Bootstrap, or TSMB for brevity.
In this work, we use block bootstrap as the time series bootstrap method in this work, but other time series bootstrap methods can be applied as well.
Proposition~\ref{thm:tsmb} shows that TSMB is a finite sample approximation of $\mathbb{E}[Y|X=x]$, the quantity we typically estimate with machine learning models in the absence of time delays.

\begin{proposition}
\label{thm:tsmb}
Assume the time delay $\bm{\delta}$ is a random variable and $f_{\bm{\delta}}(x) = \mathbb{E}[Y|X=x, \bm{\delta}]$ is the model prediction given a realized $\bm{\delta}$,
the TSMB estimator is a finite sample approximation of $\mathbb{E}[Y|X=x]$.
\end{proposition}
\begin{proof}
\unskip
\begin{align}
  \hat{y} &= \mathbb{E}[Y|X=x] \nonumber\\
    &= \mathbb{E}_{\bm{\delta}}[\mathbb{E}[Y|X=x, \bm{\delta}]] &&\text{by the law of total expectation} \nonumber\\
    &= \int_{\bm{\delta}}\mathbb{E}[Y|X=x, \bm{\delta}] \mathbf{P}(\bm{\delta}) d\bm{\delta} \nonumber\\
    &= \lim_{B\to\infty}\frac{1}{B} \sum_{b=1}^B \mathbb{E}[Y|X=x, \bm{\hat{\delta}}^b] &&\text{where } \bm{\hat{\delta}}^b \sim \mathcal{F}^B_{\bm{\delta}} \nonumber\\
    & &&  \text{the bootstrap distribution} \label{eq:tsmb_boot_approx}\\
    & \approx \frac{1}{B} \sum_{b=1}^B \mathbb{E}[Y|X=x, \bm{\hat{\delta}}^b] \label{eq:finite_tsmb_boot_approx}\\
    & =\frac{1}{B} \sum_{b=1}^B f_{\bm{\hat{\delta}}^b}(x) &&\text{the TSMB estimator} \nonumber
\end{align}
\end{proof}

\begin{table*}[ht]
  \begin{tabular}{l|cccc}
    \hline
    \textbf{Dataset} & \textbf{Size} & \textbf{\# Time Delays} & \textbf{Time Delay Type} & \textbf{Task}\\
    \hline
    \hline
    Occupancy - Fixed & 20,560 & 5 & Fixed & Classification\\
    Occupancy - Stochastic & 20,560 & 5 & Stochastic & Classification\\
    Water Pump Maintenance - Fixed & 220,320 & 6 & Fixed & Classification\\
    Water Pump Maintenance - Stochastic & 220,320 & 6 & Stochastic & Classification\\
    Power Demand - Fixed & 1,096 & 24 & Fixed & Classification\\
    Power Demand - Stochastic & 1,096 & 24 & Stochastic & Classification\\
    Air Quality - Fixed & 9,357 & 8 & Fixed & Regression\\
    Air Quality - Stochastic & 9,357 & 8 & Stochastic & Regression\\
    Mineral Processing & 1.2M/994 ($\bm{X}$/$\bm{y}$) & 10 & Unknown & Regression\\
    \hline
  \end{tabular}
  \vspace{1em}
  \caption{Overview of datasets used in the experiments.
  For the occupancy, water pump maintenance, and air quality datasets, time delays are introduced in two ways: a consistent "fixed" delay and a "stochastic" delay where time series data is adjusted at each timestamp based on random draws from a set of possible delays. The mineral processing dataset inherently has unknown delays due to the intricacies in its processes. Further details on the injected time delays can be found in Table~\ref{tab:dataset_detail} in the appendix.}
  \label{tab:datasets}
\end{table*}

The approximation presented in Equation~\ref{eq:tsmb_boot_approx} and \ref{eq:finite_tsmb_boot_approx} is a common assumption in bootstrap literature and is justified by the Dvoretzky–Kiefer–Wolfowitz inequality~\citep{massart1990tight}, which states that the empirical distribution function converges uniformly to the true distribution function at exponential rate in probability.

\section{Experiment}
\label{sec:experiment}
In this section, we present our empirical evaluation of TSMB against classic TDE methods
on nine different real-world datasets.
We then perform diagnostic analysis to examine a number of properties displayed by the TSMB estimator.
The code to reproduce the result is available at \url{https://github.com/HenryWang-000/TSMB}.

\subsection{Experiment Setup}
\subsubsection{Datasets}
\label{sec:datasets}
We assess nine diverse real-world predictive tasks, including six classification and three regression tasks. 
These datasets originate from various domains where sensor data is prevalent: occupancy detection~\citep{candanedo2016accurate}, water pump maintenance~\citep{waterpump2021}, air quality monitoring~\citep{de2008field}, power demand prediction~\citep{
dau2019ucr, keogh2006intelligent} and our privately collected mineral processing dataset.

Some of these datasets originally exhibit no apparent time delays. 
We introduce time delays in two distinct manners to encapsulate different scenarios one might encounter in practical applications. 
We first introduce fixed noises into these datasets by manually injecting a consistent time delay for each variate (i.e., feature) to simulate environments where delays are predictable and consistent over time.
On the other hand, to represent more volatile and dynamic systems, we introduce stochastic noises. In this case, at each time point, we draw from a set of five possible time delays, adjusting the time series data uniquely for each specific timestamp.
By introducing both constant and stochastic noises, we examine each method's robustness in handling both predictable and unpredictable time delay scenarios.

Table~\ref{tab:datasets} describes all datasets used in this work and Table~\ref{tab:dataset_detail} in Appendix~\ref{sec:dataset_detail} details the exact time delays injected in each datasets.

\textbf{Occupancy Detection}:
In smart environments, this dataset~\citep{candanedo2016accurate} predicts room occupancy based on variables like temperature, humidity, light, and $CO_2$. Time delays, both fixed and stochastic, are introduced to simulate real-world variability in data acquisition and processing.

\textbf{Water Pump Maintenance}:
Sourced from a public dataset~\citep{waterpump2021}, this dataset aims to determine the operational state of water pumps: normal, recovering, or broken. The dataset is adjusted for inherent imbalance by grouping ``recovering'' and ``broken'' instances together. The six most critical features are selected based on a random forest model analysis. To reflect realistic operational conditions, fixed and stochastic time delays are incorporated into the data.

\textbf{Italy Power Demand}:
This is from a public dataset~\citep{keogh2006intelligent} from the UCR~\citep{dau2019ucr} dataset, this collection includes 24 variables of unknown characteristics and time series data representing electricity demand from Italy.
Fixed and stochastic time delays are artificially injected to this dataset to create two different experimental datasets.

\textbf{Air Quality}:
Crucial for urban environment monitoring, this dataset provides hourly readings from various air quality sensors. The main goal is to predict CO concentrations based solely on these sensor readings, eliminating the need for traditional ground-truth validation methods. Both fixed and stochastic time delays are integrated into the dataset to simulate real-time data processing challenges.

\textbf{Mineral Processing}:
This is our privately collected dataset.
Contrasting the other datasets where time delays are artificially imposed, the mineral processing dataset naturally presents unknown time delays. 
The inherent unpredictability in procedures such as grinding underscores the complexity: grinding mills process ore particles over variable durations, releasing them based on specific size thresholds. 
Determining exact grinding times for individual particles is a considerable challenge, pointing to the presence of stochastic time delays. 
This dataset, sourced from an operational iron concentration mill, captures myriad sensor readings, including ore feed rates and grinding mill outputs. 
The key predictive metric is the hydrocyclone's overflow particle size, indicative of the mineral's refinement quality.

\begin{figure}[ht]
  \centering
  \includegraphics[width=\linewidth]{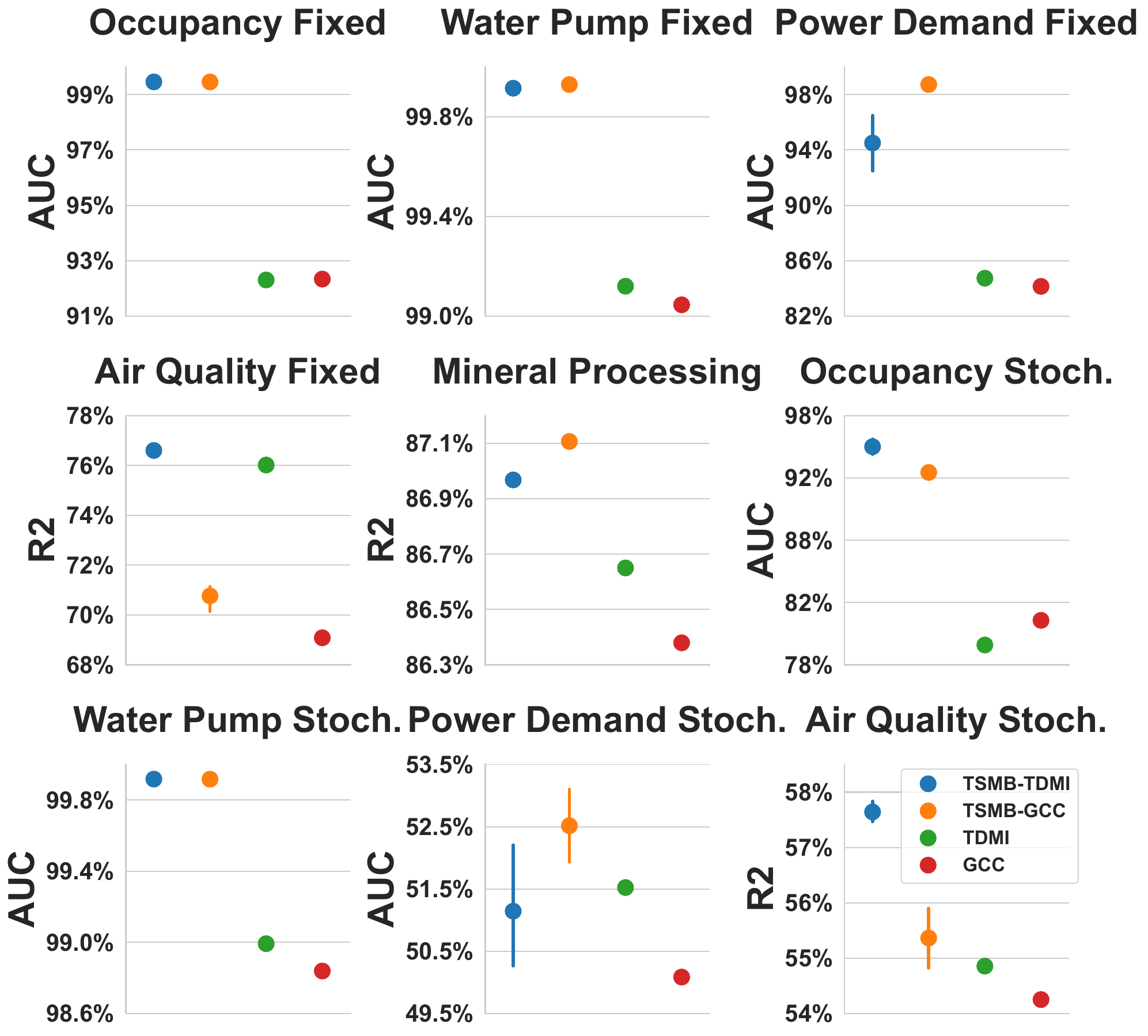}
  \caption{
Performance visualization of GBDT models applied across different datasets, showcasing the efficacy of various methods in handling time delays.
Each point indicates the AUC (for classification tasks) or $R^2$ (for regression tasks).
Across all datasets, TSMB methods consistently outperform traditional TDE techniques like TDMI and GCC.
Error bars represent 95\% CIs for TSMB-based methods and are generally small.
}
  \label{fig:mt_performance}
\end{figure}

\begin{table*}[ht]
\centering
\small
\setlength{\tabcolsep}{3pt} 
\begin{tabular}{p{2.5cm}cccccccccc}
\toprule
Method & \multicolumn{2}{c}{Occupancy} & \multicolumn{2}{c}{Pump Maintenance} & \multicolumn{2}{c}{Power Demand} & \multicolumn{2}{c}{Air Quality} & \begin{tabular}[c]{@{}c@{}}Mineral\\Processing\end{tabular} \\
& Fixed & Stochastic & Fixed & Stochastic & Fixed & Stochastic & Fixed & Stochastic & \\
\midrule
TSMB-TDMI (ours) & \textbf{0.995} & \textbf{0.950} & \textbf{0.999} & \textbf{0.999} & 0.945 & 0.511 & \textbf{0.766} & \textbf{0.571} & 0.870 \\
TSMB-GCC (ours) & \textbf{0.995} & 0.929 & \textbf{0.999} & \textbf{0.999} & \textbf{0.987} & \textbf{0.525} & 0.708 & 0.549 & \textbf{0.871} \\
TDMI & 0.923 & 0.791 & 0.991 & 0.990 & 0.847 & 0.515 & 0.760 & 0.544 & 0.867 \\
GCC & 0.923 & 0.811 & 0.990 & 0.988 & 0.841 & 0.501 & 0.691 & 0.538 & 0.864 \\
Real time delay & 0.988 & 0.988 & 0.991 & 0.991 & 0.964 & 0.964 & 1.000 & 1.000 & N/A \\
No Alignment & 0.728 & 0.722 & 0.979 & 0.979 & 0.509 & 0.519 & 0.085 & -0.106 & 0.860 \\
\bottomrule
\end{tabular}
\vspace{0.1in}
\caption{Absolute performance metrics ($R^2$ for regression tasks and AUC for classification tasks) for selected methods. For a comprehensive comparison including other variants of TSMB, refer to Table~\ref{tab:full_performance} in the appendix.}
\label{tab:main_performance}
\end{table*}

\subsubsection{Baselines}
\label{sec:baselines}
In our experiment, we benchmark the proposed TSMB method against classic TDE methods, which rely on obtaining a point estimate of the time delay.
The classic TDE baselines we consider are TDMI~\citep{Mars1982-nw, albers2012using} and GCC~\citep{Knapp1976-sr, tamim2010techniques}.
For these baselines, we first estimate the time delays with the corresponding score function, align the time series accordingly, and then fit the predictive model.
They are labeled as \textit{TDMI} and \textit{GCC} in all figures.
The proposed TSMB methods are labeled as \textit{TSMB-TDMI} and \textit{TSMB-GCC} respectively.
In the Appendix, we have further discussed a few possible variants of TSMB, where we demonstrate that while those variants offers computational trade-offs at different levels and generally perform better than traditional TDE methods, TSMB is still showing the strongest predictive performance among methods we have considered. 

We emphasize that while time delay estimation is an important component of the problem we consider, the focus of the problem we discuss in this paper is the misaligned multivariate time series predictive modeling problem, rather than the time delay estimation itself.
Therefore, when we evaluate the proposed method, we benchmark on the final model's (or a set of models') predictive performance measured in AUC or $R^2$.

\subsubsection{Experiment Setup}
Using the nine datasets summarized in Section~\ref{sec:datasets}, we evaluate TSMB against traditional TDE-based methods which rely on point estimates of the time delay.

Unless specified otherwise, we use a gradient boosted decision tree (GBDT) with 100 trees as the base predictive model for all of our experiments.
We have additionally examined TSMB's performance when used with temporal fusion transformers (TFT)~\citep{lim2021temporal} and those results are presented in Section~\ref{sec:tft_exp}, where we have observed consistent patterns compared GBDT's results, signaling the robustness of TSMB regardless the choice of the base predictive model.

We use
DIRECT as the optimization algorithm~\citep{jones1993lipschitzian, gablonsky2000locally} for all time delay optimization.
We smooth all datasets by applying moving averages to all $\bm{x}^i$ with a window size $w$ for both time delay estimation and model construction\footnote{We have also examined performing the prediction without moving average and have discovered that moving average boosts the performance of all methods, including traditional TDE baselines.}.
The moving average smoothing leads to improved performance in all methods and as a result, is used as a default setup in all experiments.
Instead of using a fixed $w$, we optimized it jointly with $\bm{\delta}$ in all time delay optimization to adapt to the distinctive characteristics of each dataset and method.
For all sampling-based methods, we use a sample size (e.g., $B$ in bootstrap-based methods) of 100. Despite that we use $B=100$, as demonstrated in Section~\ref{sec:choice_of_b}, similar performance can be achieved with a significantly smaller $B$ such as  $B=5$, mitigating the concern around computational cost.
We use block bootstrap~\citep{lahiri1999theoretical, kunsch1989jackknife} with block size being 0.25 of the training data for all time series bootstrap-based method. This proportion is small enough to allows for enough variation in the resampled time series and is big enough to retain enough structural information when calculating TDMI or GCC.
To evaluate predictive performance, we use AUC as for classification tasks and $R^2$ for regression tasks so that evaluation metrics are generally within the similar scale across different datasets unlike other common metrics such as MSE and is insensitive to class unbalance (for AUC).
Unless specified otherwise, all reported experiment results aggregated over five replications.

\subsection{Predictive Performance}
\label{sec:predictive_perf}

Our primary experimental results are presented in Figure~\ref{fig:mt_performance} and Table~\ref{tab:main_performance}.
Here, we focus on the relative performance gains achieved by applying different time series alignment methods.
Specifically, we use the AUC or $R^2$ as the evaluation metrics for classification and regression tasks respectively.

Interestingly, for datasets where we have artificially introduced time delays, the performance of TSMB methods sometimes surpasses that of models trained with the exact time delays we introduce, as highlighted by the "real time delay" row in Table~\ref{tab:main_performance}.
This suggests the presence of inherent time delays before our artificial injection or that the existing time delays in the dataset are more complex than can be captured by a singular point estimate.
The consistent outperformance of TSMB methods over traditional TDE methods underlines the limitations of using a single point estimate for predictive tasks, as is the norm with traditional TDE methods.

\begin{figure}[ht]
  \centering
  \includegraphics[width=\linewidth]{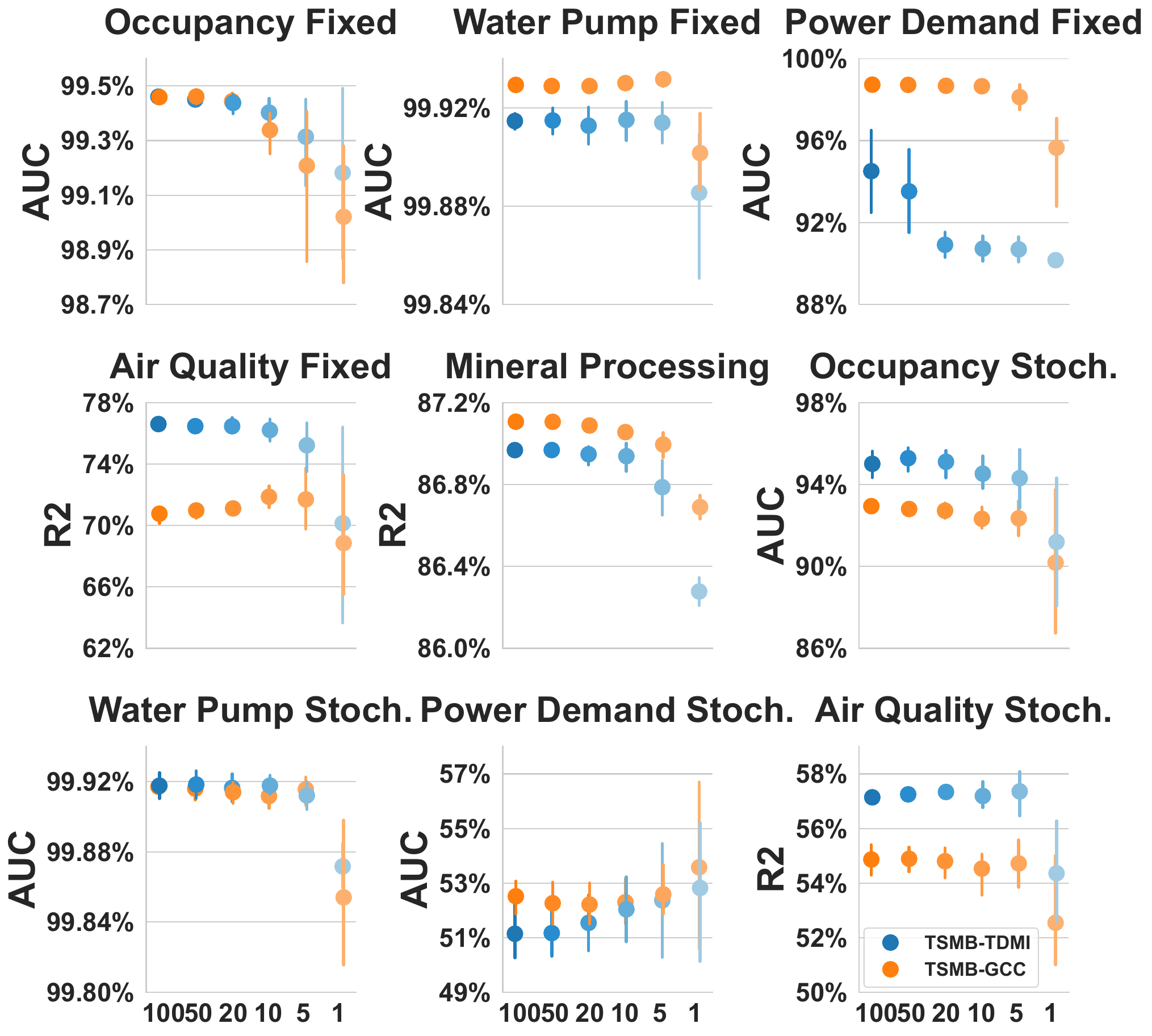}
  \caption{Ablation on bootstrap sample size $B$ for TSMB.
  The horizontal axis depicts the bootstrap sample size $B$.
  $B$ is minimally impacting the predictive performance of TSMB estimators.}
  \label{fig:abltaion_b_tsmb}
\end{figure}

\begin{figure}[ht]
  \centering
  \includegraphics[width=\linewidth]{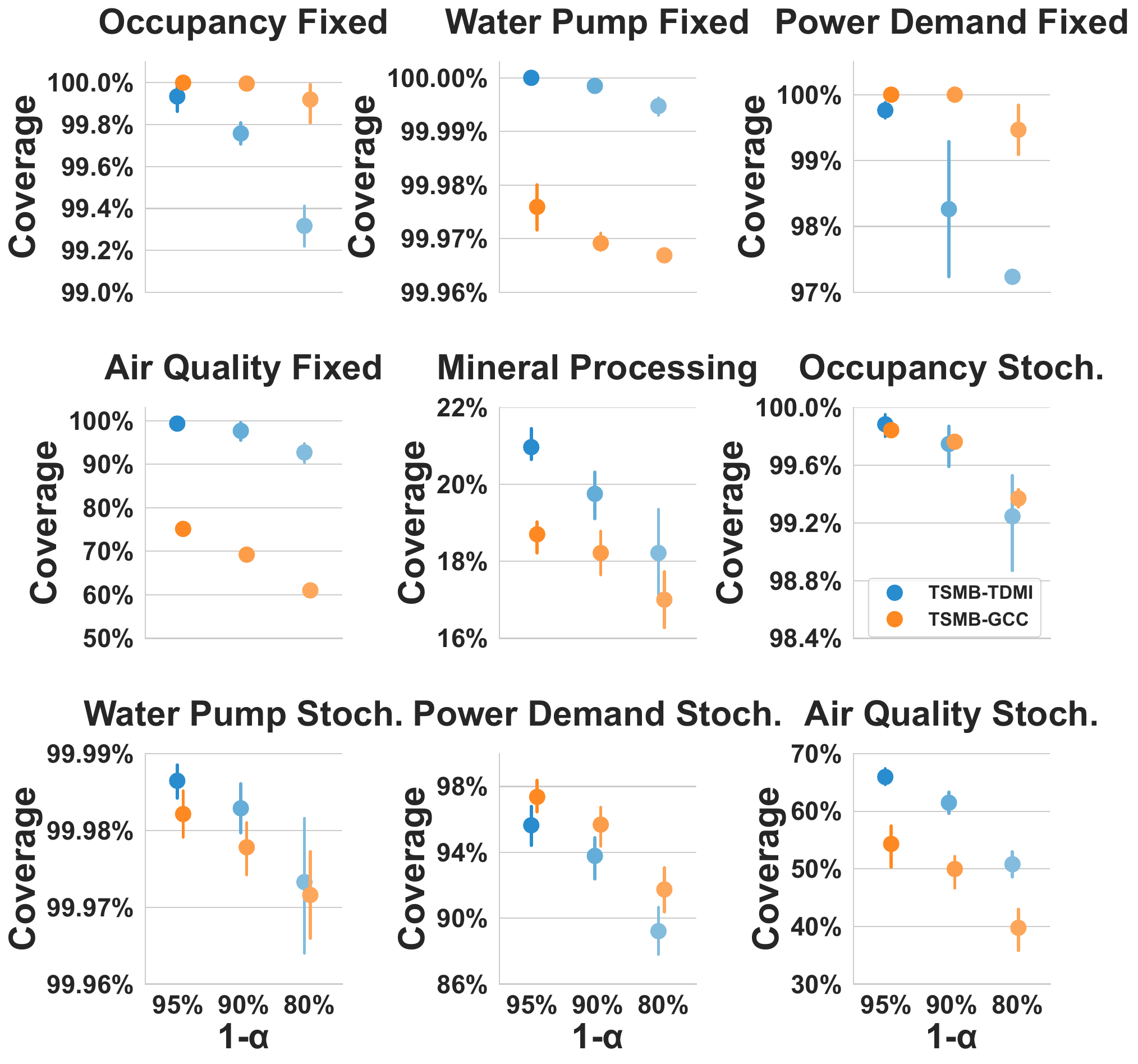}
  \caption{Bootstrap percentile $1-\alpha$ confidence interval coverage under TSMB. For classification tasks where we only observe binary values, we examine TSMB coverage using the corresponding point estimates given by TDMI or GCC.}
  \label{fig:coverage}
\end{figure}

\begin{figure*}[ht]
  \centering
  \begin{subfigure}[ht]{0.48\linewidth}
    \centering
    \includegraphics[width=\linewidth]{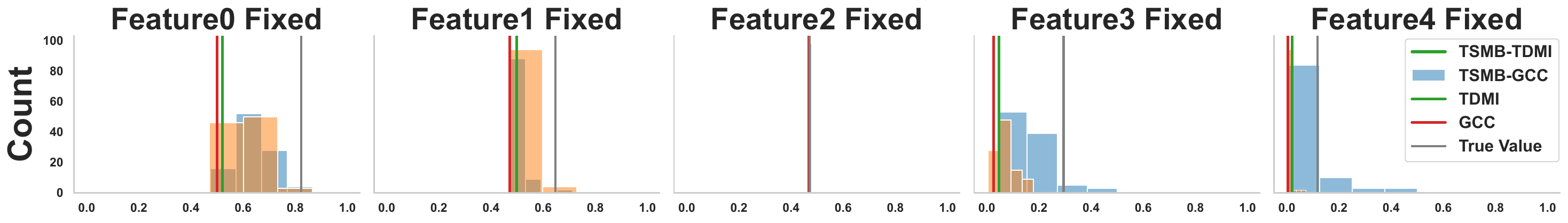}
    \includegraphics[width=\linewidth]{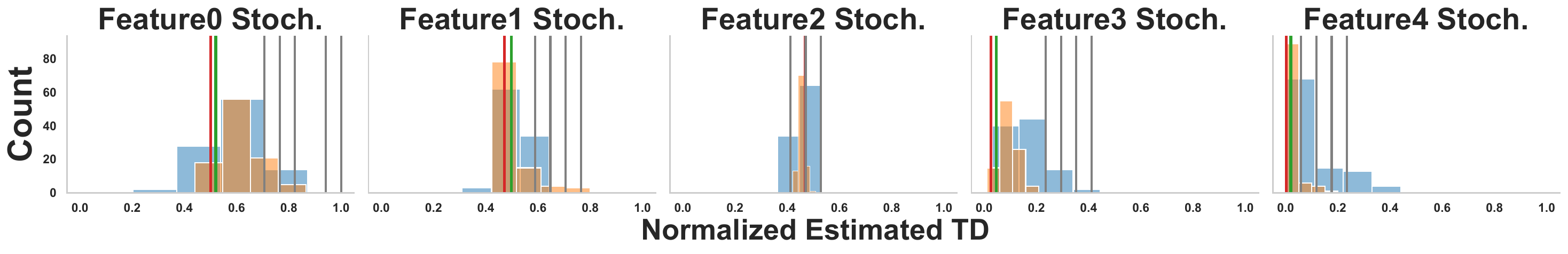}
    \caption{Occupancy time}
    \label{fig:time_delay_realdata}
  \end{subfigure}
  \begin{subfigure}[ht]{0.48\linewidth}
    \centering
    \includegraphics[width=\linewidth]{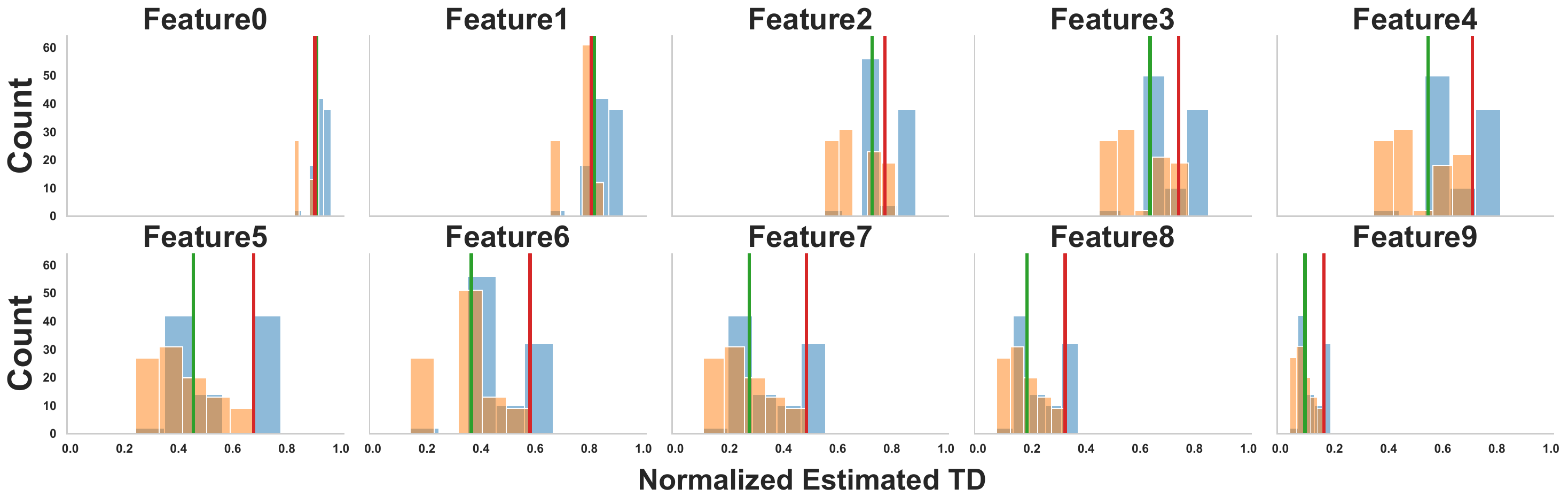}
    \caption{Mineral processing}
    \label{fig:time_delay_mineral_processing}
  \end{subfigure}
  \vspace{1em}
  \begin{subfigure}[ht]{0.45\linewidth}
    \centering
    \includegraphics[width=\linewidth]{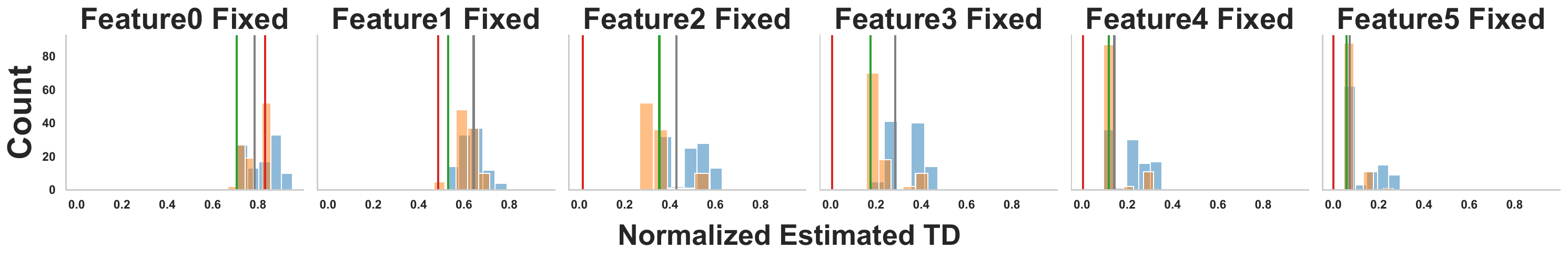}
    \includegraphics[width=\linewidth]{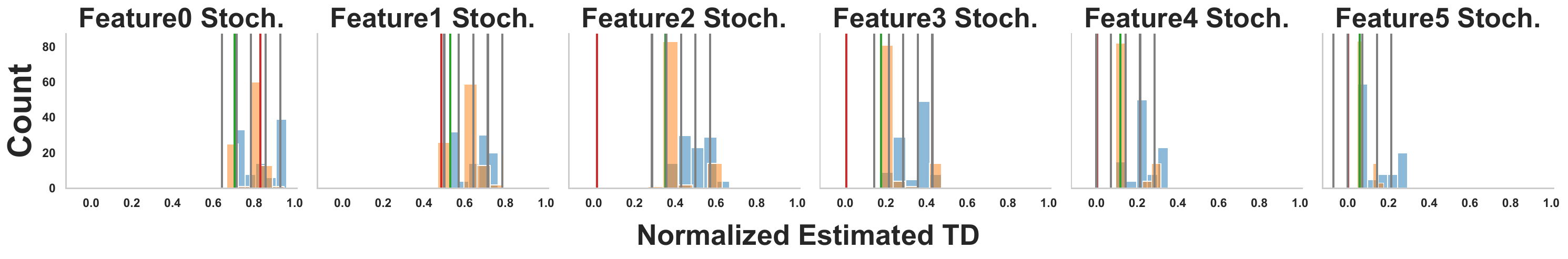}
    \caption{Water pump maintenance}
    \label{fig:time_delay_water_pump}
  \end{subfigure}
  \begin{subfigure}[ht]{0.525\linewidth}
    \centering
    \includegraphics[width=\linewidth]{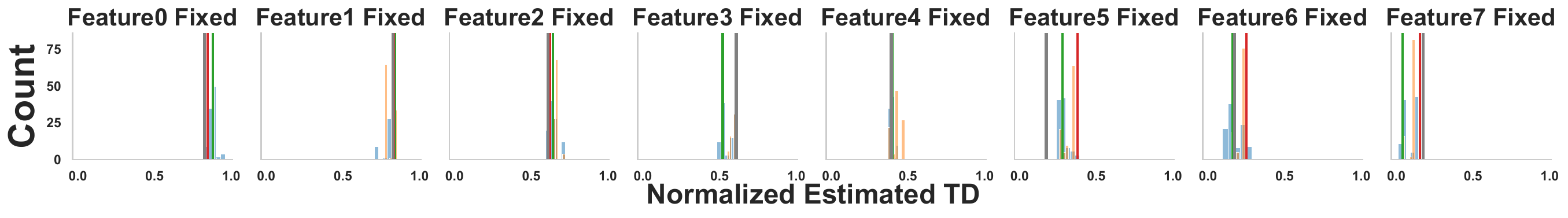}
    \includegraphics[width=\linewidth]{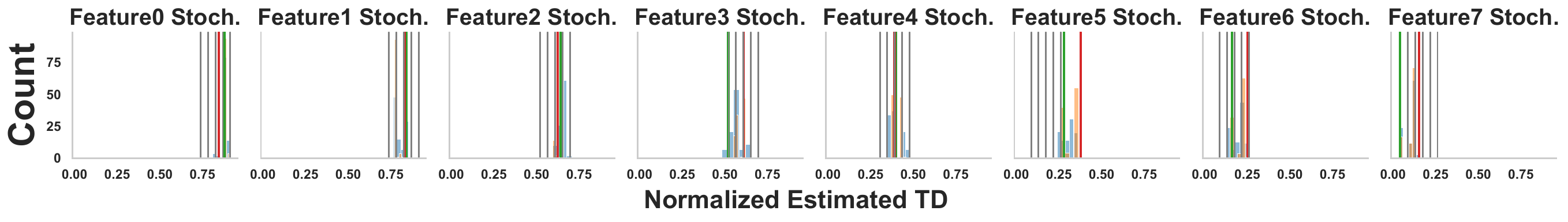}
    \caption{Air quality}
    \label{fig:time_delay_air_quality}
  \end{subfigure}
  \vspace{1em}
  \begin{subfigure}[ht]{\linewidth}
    \centering
     \includegraphics[width=0.475\linewidth]{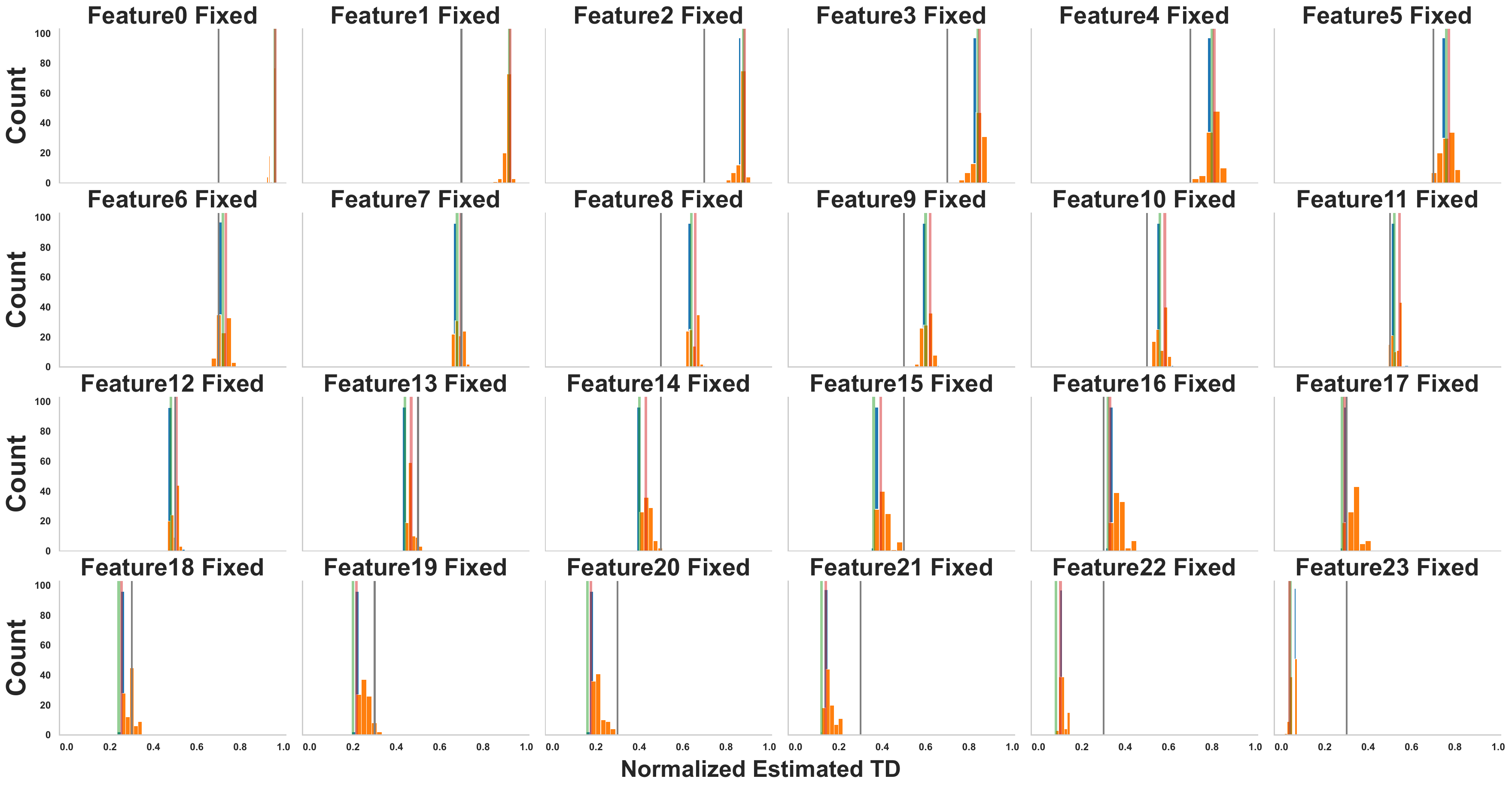}
    \includegraphics[width=0.475\linewidth]{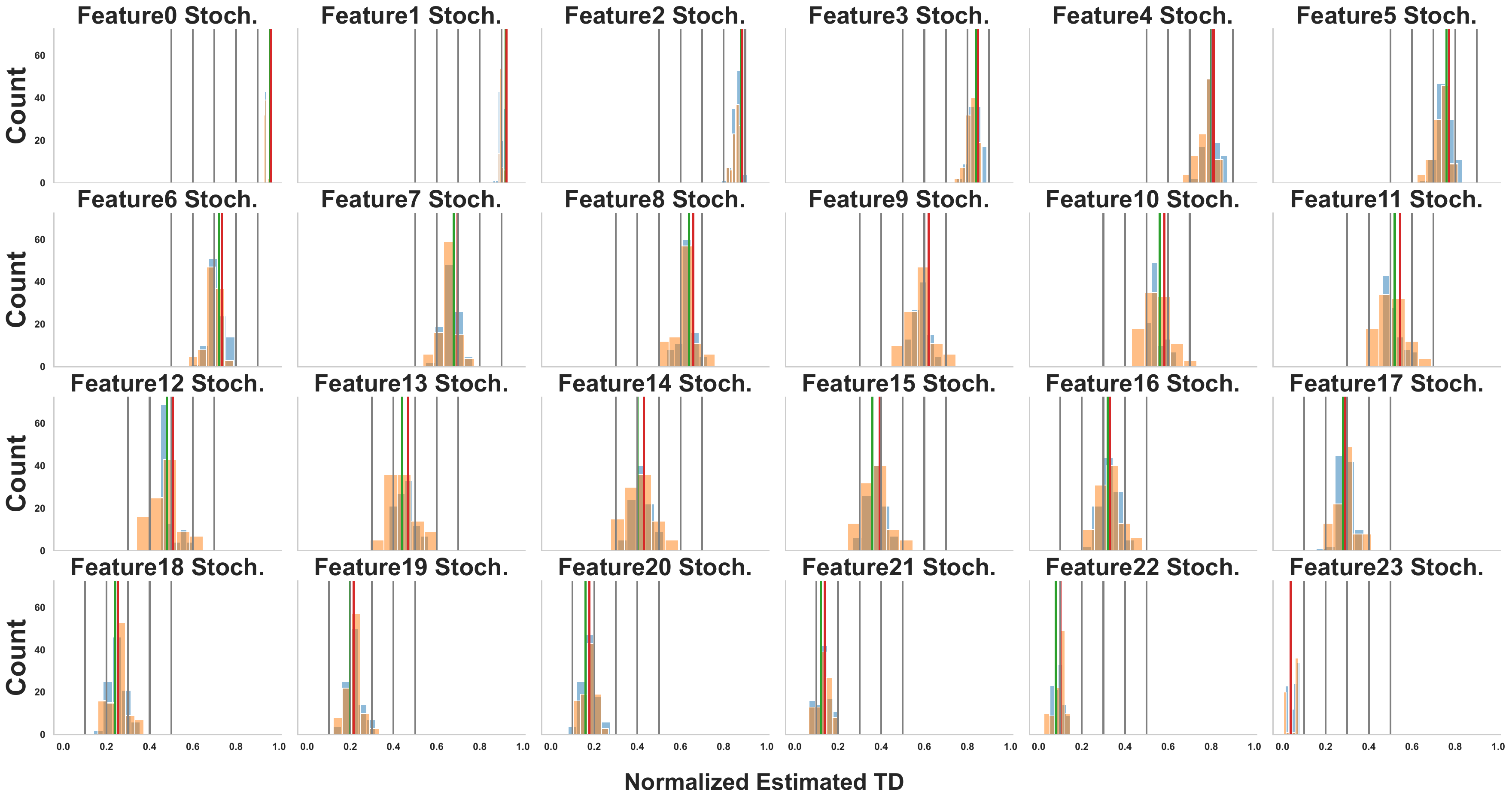}
    \caption{Power demand}
    \label{fig:time_delay_power_demand}
  \end{subfigure}
  \caption{Bootstrap distribution of normalized estimated time delays using TDMI and GCC as the score function respectively. Vertical lines represent the point estimates given by optimizing TDMI and GCC as well as the ground truth (when applicable) time delays from each dataset. The fact that many of these distributions are not even close to point distirbution suggests that there exists a significant amount of uncertainty in the estimated time delays that is being ignored by traditional TDE methods.
  }
  \label{fig:time_delay_dist}
\end{figure*}
\subsection{Choice of Bootstrap Sample Size}
\label{sec:choice_of_b}
The choice of bootstrap sample size $B$ is a determining factor of TSMB' performance and computational cost.
By varying the number of bootstrap sample size $B$, we show that even with a small number of bootstrap samples such as $B=5$, we are still able to obtain impressive predictive quality via an ablation study on the choice of $B$.
In Figure~\ref{fig:abltaion_b_tsmb}, we observe that with the decreasing value of $B$, the impact on predictive performance can be minimal.
Compared to $B=100$, a very small $B$ such as $B=5$ can still result in empirically similar predictive performance for essentially all problems at a $\frac{1}{20}$ cost, significantly alleviating the computational demand of TSMB.

\subsection{Prediction Coverage}
The empirical distribution of model predictions by TSMB is the approximated distribution of $\mathbb{E}[Y|X=x]$ under different $\bm{\delta}$.
One way to examine the validity of this approximation is to check the coverage of the model's prediction.
Plotted 95\% confidence intervals of the empirical coverage are obtained via 5 replications of the experiment.
By having an empirical bootstrap confidence interval at $1-\alpha$ level, the ground truth distribution would cover the observed data approximately $1-\alpha$ of the time.
Figure~\ref{fig:coverage} shows the empirical coverage of TSMB estimators on the test sets of all datasets.
In some dataset such as the occupancy and power demand dataset, the coverage is sensible and close to the expected values.
However, in some other datasets, both TSMB-GCC and TSMB-TDMI produce overly conservative credible intervals resulting in lower-than-expected coverage, despite their excellent predictive capabilities as we observed in Section~\ref{sec:predictive_perf}.
This result suggests that when we need to use the estimated distribution produced by TSMB such as when doing risk analysis, it is important to validate its coverage and it is useful to further investigate how one could robustly obtain well-calibrated TSMB credible intervals.

\subsection{Bootstrap Distribution Diagnostic Analysis}
\label{sec:dist_diagnostic}
While we are primarily interested in the improving predictive modeling performance with TSMB,
the predictions we obtain come from time delay bootstrap samples.
Therefore, it is also useful to understand the characteristics of the time delay bootstrap distribution.
Figure~\ref{fig:time_delay_dist} plots the empirical distribution of time delay bootstrap samples (blue and orange histograms, normalized to be between 0 and 1) as well as the point estimates made by the classic TDMI and GCC TDE methods (green and red lines).
For datasets in which we injected artificial time delays, the ground truth real time delays are also plotted as grey lines.
On many occasions, the time delays estimated by TDMI and GCC (in red and green lines) are far from the ground truth time delays, confirming the motivating example that certain types of data might pose significant challenges for estimating time delays with traditional TDE methods.
On the contrary, instead of a point estimate, TSMB yields a spectrum of time delays forming a distribution. Almost in all cases visualized here, both the ground truth and traditional TDE estimated time delays fall within the range of TSMB time delay distribution with non-trivial density.
This result implicates that the TSMB produces a coherent distribution that can nicely explain both the ground truth time delay as well as the estimated time delay by classic TDE methods in that all these point estimates are essentially manifestations of the underlying dataset, which itself can be viewed as a sample drawn from a prior distribution of time series datasets.

\subsection{TSMB Results with Temporal Fusion Transformer}
\label{sec:tft_exp}

\begin{table*}[ht]
\centering
\resizebox{\textwidth}{!}{%
\setlength{\tabcolsep}{1.5pt} %
\begin{tabular}{ccccccccccc}
\toprule
Method w/ TFT as base model & \multicolumn{2}{c}{Occupancy} & \multicolumn{2}{c}{Pump Maintenance} & \multicolumn{2}{c}{Power Demand} & \multicolumn{2}{c}{Air Quality} & \begin{tabular}[c]{@{}c@{}}Mineral\\Processing\end{tabular} \\
& Fixed & Stochastic & Fixed & Stochastic & Fixed & Stochastic & Fixed & Stochastic & \\
\midrule
TSMB-TDMI & \textbf{0.994 (± 0.001)} & \textbf{0.991 (± 0.001)} & \textbf{0.956 (± 0.001)} & \textbf{0.959 (± 0.001)} & 0.846 (± 0.034) & \textbf{0.549 (± 0.003)} & 0.766 (± 0.002) & 0.570 (± 0.007) & \textbf{0.825 (± 0.001)} \\
TSMB-GCC & 0.991 (± 0.001) & 0.985 (± 0.001) & 0.919 (± 0.054) & 0.953 (± 0.003) & \textbf{0.969 (± 0.002)} & \textbf{0.547 (± 0.005)} & 0.705 (± 0.002) & \textbf{0.582 (± 0.005)} & 0.820 (± 0.001) \\
TDMI & 0.981 & 0.958 & 0.521 & 0.519 & 0.673 & 0.543 & \textbf{0.776} & \textbf{0.587} & 0.804 \\
GCC & 0.975 & 0.955 & 0.766 & 0.779 & 0.730 & 0.534 & 0.561 & 0.619 & 0.729 \\
Real time delay & 0.984 & 0.984 & 0.534 & 0.534 & 0.957 & 0.957 & -0.114 & -0.114 & N/A \\
No Alignment & 0.749 & 0.758 & 0.766 & 0.925 & 0.480 & 0.484 & 0.047 & -0.046 & 0.790 \\
\bottomrule
\end{tabular}
}
\vspace{0.1in}
\caption{Absolute performance ($R^2$ for regression tasks and AUC for classification tasks) using Temporal Fusion Transformers (TFT) as base models. TSMB-based methods still generally outperform traditional TDE methods.}
\label{tab:tft_performance}
\end{table*}

Throughout our experiments,
we have been using GBDT as the base model in all of our evaluation as GBDT is generally well-performing across a wide variety of tasks and TSMB is a model-agnostic method.
To examine how TSMB can contribute to more complex deep models that can in theory implicitly account for time delays, we have reproduced our main experimental results with the Temporal Fusion Transformer (TFT)~\citep{lim2021temporal} model.

Table~\ref{tab:tft_performance} shows the predictive performance of TFT when trained using TSMB compared with other TDE baselines.
Figure~\ref{fig:tft_tsmb_b} shows TSMB's predictive performance with varying number of bootstrap samples $B$.
Both results demonstrate TSMB is able to improve the base model's performance empirically even if the model (TFT) can implicitly account for the time delays.

\begin{figure}[ht]
  \centering
  \includegraphics[width=0.95\linewidth]
  {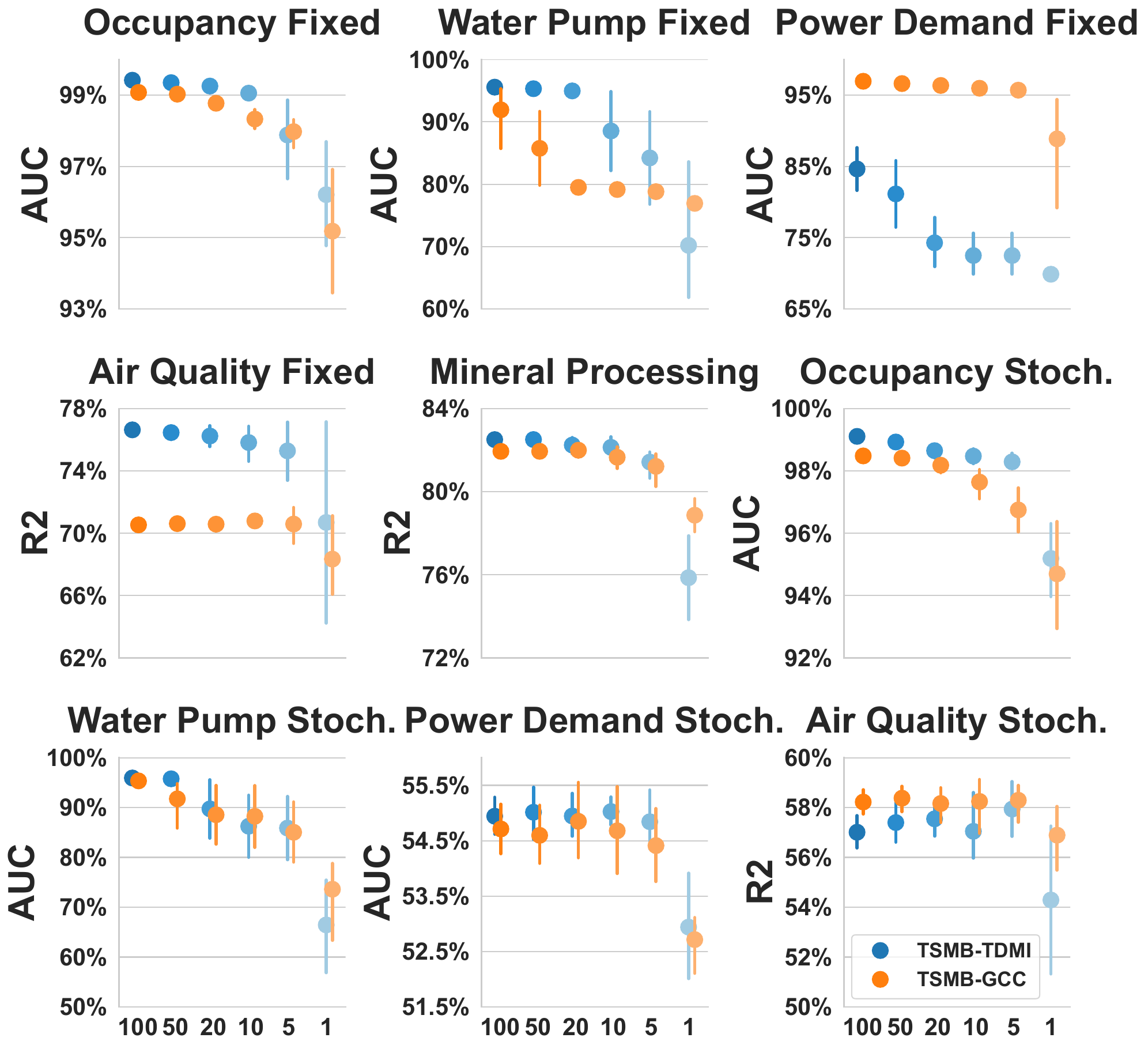}
  \caption{Performance of Temporal Fusion Transformer (TFT) model under different values of B.}
  \label{fig:tft_tsmb_b}
\end{figure}

\balance
\section{Discussion}
TSMB's training process, while effective, is computationally intensive due to the iterative optimization required over bootstrapped datasets.
While straightforward improvement such as selecting a smaller bootstrap sample size can help manage the computational load as studied in Section~\ref{sec:choice_of_b},
there are other strategies that can be leveraged to further reduce the computational demands.

Firstly, bootstrap samples as well as downstream time delay identification and model training are performed independently in TSMB.
Therefore, they can all be trivially parallelized on modern hardware.
Secondly, time delay identification and model training are optimization problems.
For time delay identification, we aim to identify time delays that maximize a specific score function (e.g., GCC or TDMI).
For model training, we aim to identify model parameters that minimize the model’s loss function (e.g., MSE or cross-entropy).
Given that in TSMB we are essentially solving the same optimization problems with slightly different datasets resulted from time series bootstrap, the time delays and model parameters identified in each replication of TSMB are likely to be similar, as empirically verified in Figure~\ref{fig:time_delay_dist}.
Therefore, once a solution for the time delay in one bootstrap sample is obtained, it can serve as an initialization heuristic for the other $B-1$ time delay identification optimization problems. Similarly, this warm-starting idea can be applied to model training.
Finally, computational efficiency can be further enhanced by subsampling the dataset before applying TSMB. However, it is important to note that the sensitivity to subsampling may vary depending on the model and the dataset. For instance, while model training might require the entire dataset to achieve a high-quality model, the time delay estimation step might be less sensitive to subsampling.

Additionally, we have discussed a few options to computationally efficiently approximate TSMB in Appendix~\ref{sec:tsmb_variants} that can be applied in appropriate use cases.

Aside from room for computational efficiency improvement, the calibration of TSMB's predictions, as illustrated by Figure~\ref{fig:coverage}, presents opportunities for refinement. How to construct a better calibrated predictive model using TSMB that can be used for decision making could be an useful future research direction.

\section{Conclusion}
In this work, we revisited the challenge of time delay estimation and how elegantly handling potentially stochastic time delays can have a significant impact on the accuracy and reliability of predictive models across a spectrum of data-intensive applications.
Time Series Model Bootstrap, or TSMB, is introduced as a robust solution to the inherent complexities and unpredictable dynamics characteristic of multivariate time series data. Through rigorous experimentation, TSMB has demonstrated its capacity to enhance prediction accuracy and robustness significantly.

As we advance into a future shaped by an ever-growing reliance on data-driven systems, the relevance and potential of TSMB in addressing the nuances of time delay in predictive modeling become increasingly evident.
We hope that the insights and methodologies presented in this work will not only enhance current practices but also inspire continued innovation, fostering advancements in predictive modeling techniques across diverse and evolving domains.

\clearpage
\bibliographystyle{ACM-Reference-Format}
\balance
\bibliography{references, paperpile}

\clearpage
\appendix
\section*{Appendix}

\begin{table*}[ht]
\resizebox{0.9\textwidth}{!}{
  \begin{tabular}{l|cccc}
    \hline
    \textbf{Dataset} & \textbf{Sampling Freq.} & \textbf{Search Range} & \textbf{Injected Time Delay} & \textbf{Train/Val/Test Split} \\
    \hline
    \hline
    Occupancy - Fixed & 1min & [10, 180]min & [150, 120, 90, 60, 30]min & 25\%/25\%/50\% \\
    \cline{1-5}
    \multirow{5}{*}{Occupancy - Stochastic} & \multirow{5}{*}{1min} & \multirow{5}{*}{[10, 180]min} & Possible TD 1: [180, 140, 100, 80, 50]min & \multirow{5}{*}{25\%/25\%/50\%} \\
    & & & Possible TD 2: [170, 130, 100, 70, 40]min & \\
    & & & Possible TD 3: [150, 120, 90, 60, 30]min & \\
    & & & Possible TD 4: [140, 110, 80, 50, 20]min & \\
    & & & Possible TD 5: [130, 100, 70, 40, 20]min & \\
    \cline{1-5}
    Pump Maintenance - Fixed & 1min & [10, 80]min & [65, 55, 40, 30, 20, 15]min & 50\%/N/A/50\% \\
    \cline{1-5}
    \multirow{5}{*}{Pump Maintenance - Stochastic} & \multirow{5}{*}{1min} & \multirow{5}{*}{[10, 80]min} & Possible TD 1: [75, 65, 50, 40, 30, 25]min & \multirow{5}{*}{50\%/N/A/50\%} \\
    & & & Possible TD 2: [70, 60, 45, 35, 25, 20]min & \\
    & & & Possible TD 3: [65, 55, 40, 30, 20, 15]min & \\
    & & & Possible TD 4: [60, 50, 35, 25, 15, 10]min & \\
    & & & Possible TD 5: [55, 45, 30, 20, 10, 5]min & \\
    \cline{1-5}
    Power Demand - Fixed & 1min & [0, 10]min & [7,7,7,7,7,7,7,7,5,5,5,5,5,5,5,5,3,3,3,3,3,3,3,3]min & 6\%/N/A/94\% \\
    \cline{1-5}
    \multirow{5}{*}{Power Demand - Stochastic} & \multirow{5}{*}{1min} & \multirow{5}{*}{[0, 10]min} & Possible TD 1: [9,9,9,9,9,9,9,9,7,7,7,7,7,7,7,7,5,5,5,5,5,5,5,5]min & \multirow{5}{*}{6\%/N/A/94\%} \\
    & & & Possible TD 2: [8,8,8,8,8,8,8,8,6,6,6,6,6,6,6,6,4,4,4,4,4,4,4,4]min & \\
    & & & Possible TD 3: [7,7,7,7,7,7,7,7,5,5,5,5,5,5,5,5,3,3,3,3,3,3,3,3]min & \\
    & & & Possible TD 4: [6,6,6,6,6,6,6,6,4,4,4,4,4,4,4,4,2,2,2,2,2,2,2,2]min & \\
    & & & Possible TD 5: [5,5,5,5,5,5,5,5,3,3,3,3,3,3,3,3,1,1,1,1,1,1,1,1]min & \\
    \cline{1-5}
    Air Quality - Fixed & 1hr & [1, 24]hr & [20, 20, 15, 15, 10, 5, 5, 5]hr & 50\%/25\%/25\% \\
    \cline{1-5}
    \multirow{5}{*}{Air Quality - Stochastic} & \multirow{5}{*}{1hr} & \multirow{5}{*}{[1, 24]hr} & Possible TD 1: [22, 22, 17, 17, 12, 7, 7, 7]hr & \multirow{5}{*}{50\%/25\%/25\%} \\
    & & & Possible TD 2: [21, 21, 16, 16, 11, 6, 6, 6]hr & \\
    & & & Possible TD 3: [20, 20, 15, 15, 10, 5, 5, 5]hr & \\
    & & & Possible TD 4: [19, 19, 14, 14, 9, 4, 4, 4]hr & \\
    & & & Possible TD 5: [18, 18, 13, 13, 8, 3, 3, 3]hr & \\
    \cline{1-5}
    Mineral Processing & 7s/2hr ($\bm{X}$/$\bm{y}$) & [0, 90]min & N/A & 50\%/25\%/25\% \\
    \hline
  \end{tabular}
}
  \vspace{0.5em}  
  \caption{Experimental dataset details.}
  \label{tab:dataset_detail}
\end{table*}

\begin{table*}[htb]
\centering
\resizebox{0.9\textwidth}{!}{%
\setlength{\tabcolsep}{1.5pt} %
\begin{tabular}{ccccccccccc}
\toprule
Method & \multicolumn{2}{c}{Occupancy} & \multicolumn{2}{c}{Pump Maintenance} & \multicolumn{2}{c}{Power Demand} & \multicolumn{2}{c}{Air Quality} & \begin{tabular}[c]{@{}c@{}}Mineral\\Processing\end{tabular} \\
& Fixed & Stochastic & Fixed & Stochastic & Fixed & Stochastic & Fixed & Stochastic & \\
\midrule
TSMB-TDMI (ours) & \textbf{0.995(± 0.000)} & \textbf{0.950(± 0.007)} & \textbf{0.999(± 0.000)} & \textbf{0.999(± 0.000)} & 0.945(± 0.021) & 0.511(± 0.010) & \textbf{0.766(± 0.002)} & 0.571(± 0.002) & 0.870(± 0.000) \\
TSMB-GCC (ours) & \textbf{0.995(± 0.000)} & 0.929(± 0.002) & \textbf{0.999(± 0.000)} & \textbf{0.999(± 0.000)} & \textbf{0.987(± 0.000)} & 0.525(± 0.006) & 0.708(± 0.005) & 0.549(± 0.006) & \textbf{0.871(± 0.000)} \\
TDB-TDMI (TSMB variant) & 0.915(± 0.005) & 0.806 (± 0.009) & 0.993 (± 0.001) & 0.993 (± 0.001) & 0.848 (± 0.000) & 0.549 (± 0.012) & \textbf{0.764 (± 0.004)} & 0.548 (± 0.014) & 0.866 (± 0.001) \\
TDB-GCC (TSMB variant) & 0.851 (± 0.011) & 0.796 (± 0.010) & 0.992 (± 0.000) & 0.991 (± 0.000) & 0.964 (± 0.000) & 0.519 (± 0.016) & 0.634 (± 0.002) & 0.532 (± 0.010) & 0.867 (± 0.001) \\
Perturbed-TDMI (TSMB variant) & 0.951 (± 0.005) & 0.943 (± 0.002) & \textbf{0.999(± 0.000)} & \textbf{0.999(± 0.000)} & 0.941 (± 0.003) & \textbf{0.568 (± 0.001)} & \textbf{0.765 (± 0.001)} & 0.581 (± 0.003) & 0.868 (± 0.000) \\
Perturbed-GCC (TSMB variant) & 0.951 (± 0.004) & 0.942 (± 0.001) & \textbf{1.000 (± 0.000)} & \textbf{1.000 (± 0.000)} & 0.964 (± 0.002) & 0.533 (± 0.002) & 0.703 (± 0.003) & \textbf{0.587 (± 0.002)} & 0.867 (± 0.000) \\
TDMI & 0.923 & 0.791 & 0.991 & 0.990 & 0.847 & 0.515 & 0.760 & 0.544 & 0.867 \\
GCC & 0.923 & 0.811 & 0.990 & 0.988 & 0.841 & 0.501 & 0.691 & 0.538 & 0.864 \\
Real time delay & 0.988 & 0.988 & 0.991 & 0.991 & 0.964 & 0.964 & 1.000 & 1.000 & N/A \\
No Alignment & 0.728 & 0.722 & 0.979 & 0.979 & 0.509 & 0.519 & 0.085 & -0.106 & 0.860 \\
\bottomrule
\end{tabular}
}
\vspace{0.5em}
\caption{Absolute performance ($R^2$ for regression tasks and AUC for classification tasks) on all methods. We do not know the real time delay for the mineral processing dataset as it is a real-world dataset with unknown time delays.
95\% confidence intervals are reported for TSMB-based methods.
For other baselines, repeated experiments result in the same metric value, hence the CI is effectively zero on the specific datasets evaluated and not omitted.}
\label{tab:full_performance}
\end{table*}

\section{Experiment Dataset Details}
\label{sec:dataset_detail}
Table~\ref{tab:dataset_detail} shows detailed information about the datasets on which our experiments are performed.
Across all datasets and baselines

\section{Possible Variants of TSMB}
\label{sec:tsmb_variants}
Performing TSMB on a large dataset with complex models can be expensive.
There are generally two components in TSMB that are computationally demanding:
optimizing for $\bm{\hat{\delta}}^b$ on  $\mathcal{D}^b$, and the construction of each bootstrap model $f_{\bm{\hat{\delta}}^b}$.
Here we consider two classes of TSMB variants that respectively allow us to construct fewer model during training (time delay bootstrap) and performing fewer time delay optimizations (perturbed model average).

\subsection{Description of TSMB Variants}
\subsubsection{Time Delay Bootstrap (TDB)}
\label{sec:tdb}
One way to approximately bypass the need of integrating over $\bm{\delta}$ values is to replace its bootstrap distribution with a single point distribution (i.e., a point estimate) centered around its expectation.
Specifically, when constructing the predictive model, we are making the following approximation:
$$ \mathbb{E}_{\bm{\delta}}[\mathbb{E}[Y|X=x, \bm{\delta}]] \approx \mathbb{E}[Y|X=x, \bm{\delta} = \mathbb{E}[\bm{\delta}]]$$
Similar ideas are often employed in empirical Bayes methods~\citep{maritz2018empirical}.
Figure~\ref{fig:time_delay_dist} shows the time delay bootstrap distribution estimated using TDMI and GCC as the score function respectively in the nine datasets we use in this paper.
We can see that the empirical bootstrap distributions are not necessarily centered around the corresponding point estimates given by traditional TDE methods, but typically the ground truth time delay (when applicable) almost always falls under regions with non-trivial mass.
Since we only use the bootstrap samples up till the estimation of $\mathbb{E}[\bm{\delta}]$, we refer this method as time delay bootstrap (TDB).

\subsubsection{Perturbed Model Average}
\label{sec:perturbed_model}
Although TDB allows us to bypass the cost of constructing $B$ bootstrap model, it still requires optimizing $\bm{\hat{\delta}}^b$ for $B$ times.
This procedure may still be costly with no obvious advantage and at the end of the day, we are still constructing a single machine learning model based on a point estimate of the time delay.
Instead of using the bootstrap samples to approximate the distribution of $\bm{\delta}$, one may simply assume that $\bm{\delta} \sim \mathcal{F}^P_{\bm{\delta}}$ for some parametric distribution $\mathcal{F}^P_{\bm{\delta}}$ that can reasonably approximate the true underlying distribution $\mathcal{F}_{\bm{\delta}}$. 
Such approximation allows us to draw essentially as many time delay samples as needed.
Following the empirical Bayes idea, we let $\mathcal{F}^P_{\bm{\delta}} = N(\bm{\mu}=\bm{\delta}_{TDE}, \bm{\Sigma} = 0.1 \mathbf{I})$ on a normalized scale where $\bm{\delta}_{TDE}$ is the point estimate\footnote{Since we normalize all time delays to be between 0 and 1, when sampling from $\mathcal{F}^P_{\bm{\delta}}$ we clip all sample values below 0 or above 1 to be 0 or 1. This makes $\mathcal{F}^P_{\bm{\delta}}$ technically a censored normal distribution.} given by traditional TDE methods such as ones obtained by maximizing TDMI or GCC.
While this approach seems less principled compared to TSMB and TDB, it allows for a similar model ensemble as in TSMB, while being able to avoid the costly repeated optimization of $\bm{\hat{\delta}}^b$.
Since we use a Gaussian distribution to perturb the estimated time delay around the point estimation given by traditional TDE methods and aggregate model predictions for each of those perturbed time delay, we refer to this method as perturbed model average.
As we will present in the experiment section, perturbed model average methods, though on average perform worse than their TSMB counterparts, still offer competitive performance gain compared to traditional TDE methods.

\begin{figure}[ht]
  \centering
  \begin{subfigure}[b]{\linewidth}
    \centering
    \includegraphics[width=0.95\linewidth]{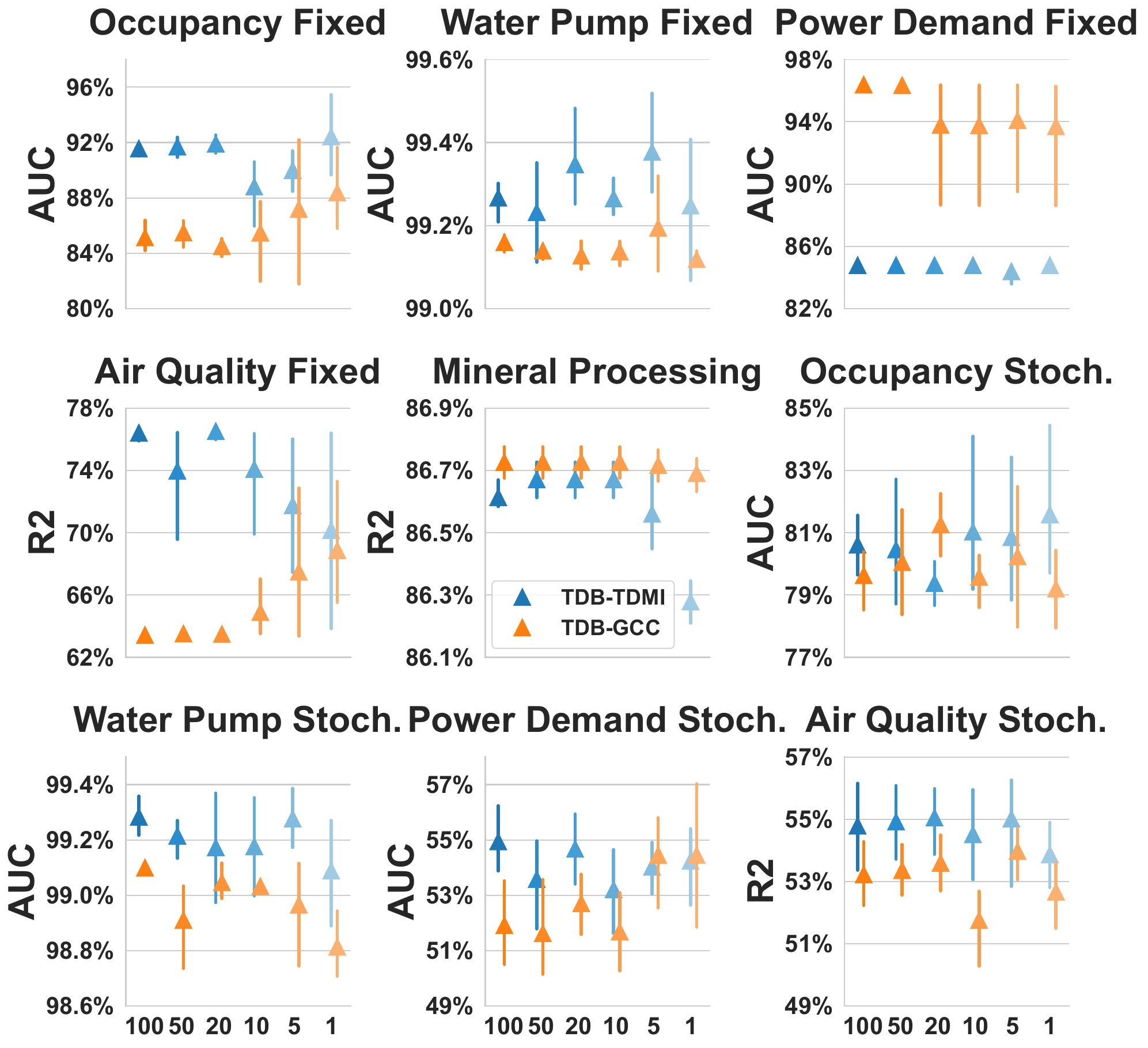}
    \caption{Time delay bootstrap (TDB)}
    \label{fig:tdb_ablation_b}
  \end{subfigure}
  \begin{subfigure}[b]{\linewidth}
    \centering
    \includegraphics[width=0.95\linewidth]{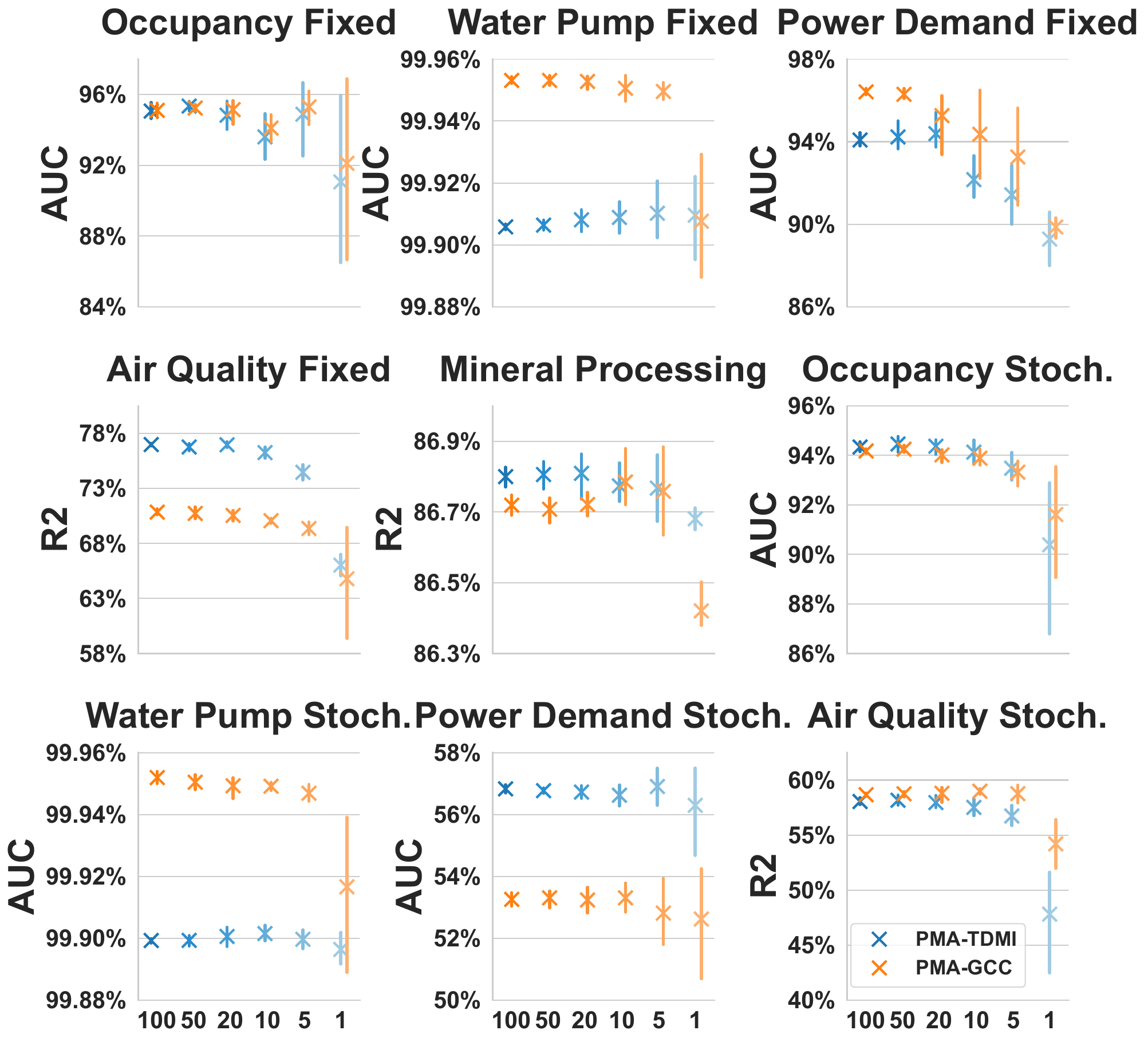}
    \caption{Perturbed model average}
    \label{fig:perturbed_ablation_b}
  \end{subfigure}
  \caption{Ablation study on the choice of bootstrap sample size $B$ for time delay bootstrap (TDB) and perturbed model average. We observe patterns similar to the ones with TSMB where using a relatively small bootstrap sample size (e.g., $B=5$) can still result in a good predictive performance, further mitigating the computational demand on these TSMB variants.}
  \label{fig:ablation_b_full}
\end{figure}

\subsection{Experimental Results with TSMB Variants}
Here, we compare variants of TSMB described in Section~\ref{sec:tdb} and \ref{sec:perturbed_model}.
We additionally evaluated the performance of TSMB variants.
Table~\ref{tab:full_performance} shows the full suite of relative predictive performance values.

We further performed ablation study on the choice of bootstrap sample size $B$ and have observed similar pattern for TDB and perturbed model average (Figure~\ref{fig:ablation_b_full}) where even a small bootstrap sample size (e.g., $B=5$) can still result in competitive predictive performance.
When a small $B$ is used together with the TSMB variants, the computational demand can be significantly alleviated.

\end{document}